\documentclass{article}

\usepackage[preprint]{neurips_2025}

\usepackage{amsmath,amsfonts,bm}

\def\eqref#1{equation~\ref{#1}}
\def\1{\bm{1}}

\DeclareMathAlphabet{\mathsfit}{\encodingdefault}{\sfdefault}{m}{sl}
\SetMathAlphabet{\mathsfit}{bold}{\encodingdefault}{\sfdefault}{bx}{n}

\usepackage{hyperref}
\usepackage{url}

\usepackage{natbib}

\usepackage{algpseudocode}
\usepackage{algorithm}
\usepackage{multirow}
\usepackage[utf8]{inputenc} % allow utf-8 input
\usepackage[T1]{fontenc}    % use 8-bit T1 fonts
\usepackage{hyperref}       % hyperlinks
\usepackage{url}            % simple URL typesetting
\usepackage{booktabs}       % professional-quality tables
\usepackage{amsfonts}       % blackboard math symbols
\usepackage{nicefrac}       % compact symbols for 1/2, etc.
\usepackage{microtype}      % microtypography
\usepackage{xcolor}         % colors
\usepackage{bbm}
\usepackage{amsmath,bm, amsthm,booktabs}
\usepackage{caption}
\usepackage{subcaption}
\usepackage{graphicx}
\usepackage{wrapfig}

\newtheorem{example}{Example}

\newtheorem{theorem}{Theorem}
\newtheorem{lemma}{Lemma}[section]

\newtheorem{corollary}{Corollary}
\newtheorem{definition}{Definition}[section]
\newtheorem{assumption}{Assumption}[section]

\title{Bayesian Risk-Sensitive Policy Optimization For MDPs With General Loss Functions
}

\author{
    Xiaoshuang Wang\\
    Industrial and Systems Engineering\\
    Georgia Institute of Technology\\
    Atlanta, GA 30332, USA\\
    \texttt{xwang3094@gatech.edu}\\
    \And
    Yifan Lin\\
    Industrial and Systems Engineering\\
    Georgia Institute of Technology\\
    Atlanta, GA 30332, USA\\
    \texttt{ylin429@gatech.edu}
    \And
    Enlu Zhou\\
    Industrial and Systems Engineering\\
    Georgia Institute of Technology\\
    Atlanta, GA 30332, USA\\
    \texttt{enlu.zhou@isye.gatech.edu}
}

\begin{document}

\maketitle

% The \author macro works with any number of authors. There are two commands
% used to separate the names and addresses of multiple authors: \And and \AND.
%
% Using \And between authors leaves it to LaTeX to determine where to break the
% lines. Using \AND forces a line break at that point. So, if LaTeX puts 3 of 4
% authors names on the first line, and the last on the second line, try using
% \AND instead of \And before the third author name.

\maketitle
\begin{abstract}
Motivated by many application problems, we consider Markov decision processes (MDPs) with a general loss function and unknown parameters. To mitigate the epistemic uncertainty associated with unknown parameters, we take a Bayesian approach to estimate the parameters from data and impose a coherent risk functional (with respect to the Bayesian posterior distribution) on the loss. Since this formulation usually does not satisfy the interchangeability principle, it does not admit Bellman equations and cannot be solved by approaches based on dynamic programming. Therefore, We propose a policy gradient optimization method, leveraging the dual representation of coherent risk measures and extending the envelope theorem to continuous cases. We then show the stationary analysis of the algorithm with a convergence rate of 
{$\mathcal{O}(T^{-1/2}+r^{-1/2})$}, where $T$  is the number of  policy gradient iterations and $r$ is the sample size of the gradient estimator. We further extend our algorithm to an episodic setting, and establish the global convergence of the extended algorithm and provide bounds on the number of iterations needed to achieve an error bound {$\mathcal{O}(\epsilon)$}  in each episode. %, \textcolor{red}{which achieves  the global convergence under the true environment.}

% Motivated by many application problems, we consider Markov decision processes (MDPs) with a general loss function and unknown parameters. To mitigate the epistemic uncertainty associated with unknown parameters, we take a Bayesian approach to estimate the parameters from data and impose a coherent risk functional (with respect to the Bayesian posterior distribution) on the general loss function. Since this formulation usually does not satisfy the interchangeability principle, it does not admit Bellman equations and cannot be solved by approaches based on dynamic programming. Therefore, we develop a policy gradient optimization approach to address this problem.  We utilize the dual representation of the coherent risk measure and extend the envelope theorem to derive the policy gradient. Our extension of the envelope theorem
% from the discrete case to the continuous case may be of independent interest. We then show the first-order convergence of the proposed algorithm with a convergence rate of 
% % $\mathcal{O}(\frac{1}{t})$, 
% \textcolor{blue}{
% $\mathcal{O}(T^{-1/2}+r^{-1/2})$, where $T$  is the number of  policy gradient iterations and $r$ is the sample number.} We further extend our algorithm to an episodic setting, and establish the consistency of the extended algorithm and provide bounds on the number of iterations needed to achieve an error bound \textcolor{blue}{$\mathcal{O}(\epsilon)$}  in each episode.
\end{abstract}

\section{Introduction}\label{sec:intro}

Markov decision process (MDP) is a paradigm for modeling sequential decision making under uncertainty, with a primary focus on identifying an optimal policy that minimizes the (discounted)  expected total cost.  However, the standard form of MDP is not sufficient for modeling some practical problems. 
% It usually focuses on finding an optimal policy that minimizes the (discounted) total expected cost.  
For example, consider a self-driving car in a dynamic urban environment. It must reach its destination efficiently yet safely amid unpredictable events, calling for a general optimization objective within the MDP framework. At the same time, it faces incomplete knowledge of its environment, such as road conditions. Thus, the decision maker confronts two challenges: defining a general performance measure for intrinsic uncertainty and handling epistemic uncertainty about the environment. This paper addresses both challenges simultaneously within the MDP framework.

% For example, consider a self-driving car operating in a dynamic urban environment. On one
% hand, 
% %some events in the environment are inherently unpredictable, such as the sudden braking of a vehicle ahead. The 
% the self-driving car must not only reach its destination efficiently but also safely against unpredictable events,  requiring a general optimization objective within the MDP framework.  On the other hand, the car has incomplete knowledge about its environment, such as road conditions. In such a case, the decision maker encounters two key challenges: the need for a general performance measure to address intrinsic uncertainty, and epistemic uncertainty about the environment. This paper is motivated by these challenges and aims to address both the general loss function and epistemic uncertainty simultaneously  in the MDP framework.

There is extensive literature addressing general loss functions and epistemic uncertainty separately. 
%In many practical problems, expected costs may not be the most appropriate metric for evaluating performance, and the decision maker may be interested in alternative objectives that better address intrinsic uncertainty. 
For instance,  risk-sensitive objectives have been explored in the contexts of MDPs  \citep{howard1972risk,ruszczynski2010risk,mannor2011mean,petrik2012an}, stochastic optimal control  \citep{borkar2002risk,moon2020generalized}, and stochastic programming  \citep{shapiro2012minimax,pichler2022risk} literature. These objectives cannot be simply represented as the total expected cost. Epistemic uncertainty arises when some MDP parameters, such as transition probabilities, are unknown and must be estimated from available data. 
%If the data used to estimate the true, but unknown, underlying MDP are insufficient, the estimated MDP may significantly deviate from the actual MDP, resulting in poor policy performance. 
This discrepancy between the estimated and true MDP is referred to as epistemic uncertainty. Numerous approaches have been proposed to address epistemic uncertainty in MDPs, with robust MDP  \citep{nilim2003robustness, iyengar2005robust, delage2010percentile, wiesemann2013robust, petrik2019beyond} being one of the most widely adopted formulations. A more flexible and less conservative formulation, coined as Bayesian Risk MDP, was recently proposed by \citet{lin2022bayesian}. 

To the best of our knowledge, this paper is the first to address Markov decision processes (MDPs) with both general loss functions and epistemic uncertainty. Existing methods for handling epistemic uncertainty in MDPs (e.g., robust or Bayesian-risk formulations) typically rely on Bellman equations, which in turn depend on an interchangeability principle. This principle usually fails when the objective is a general loss function rather than the standard cumulative total cost, making dynamic programming inapplicable. Consequently, one cannot simply wrap a robust or Bayesian-risk method around a general objective.

Motivated by this limitation, we take a policy optimization approach to MDPs with general loss functions, particularly focusing on convex loss functions in the occupancy measure. Convexity is a natural assumption—such loss functions are widely used (see \cite{pennings2003shape}) and encompass many existing models, such as risk-sensitive MDPs, constrained MDPs, and the standard expected total cost as a special case (where the loss is linear in the occupancy measure). To handle both epistemic uncertainty (unknown parameters like transition probabilities) and intrinsic uncertainty, we adopt a Bayesian perspective: we estimate parameters from a fixed batch of data and impose a coherent risk functional with respect to the Bayesian posterior distribution. This yields an offline optimization problem with a two-layer composite objective: an outer coherent risk measure and an inner convex loss function.

To solve this composite problem, we directly optimize policies and accommodate high-dimensional representations such as neural networks. The coherent risk measure admits a dual representation as shown in \cite{shapiro2021lectures}, which can be expressed as a supremum over a risk envelope set. Building on this, we extend the envelope theorem in \cite{milgrom2002envelope} to derive the policy gradient. Unlike \cite{tamar2015policy}, whose method was limited to a discrete parameter space, our extension covers the continuous case and may be of independent interest. The resulting policy gradient involves the gradient of the loss with respect to policy parameters, which we estimate using the recent variational method proposed in \cite{zhang2020variational}. Alternative estimators, such as the zeroth-order method of \cite{balasubramanian2022zeroth}, are also applicable. Incorporating these estimators into the policy gradient allows us to perform gradient descent on the composed objective. Finally, to enhance practicality, we extend our method to the episodic setting, where the agent iteratively gathers new data under the current policy and refines its estimates and policy updates accordingly.

There have been efforts to solve some special cases of our composite problem using policy gradient algorithms. For example, \cite{chow2014algorithms} applied Conditional Value-at-Risk(CVaR) to the total cost and developed policy gradient and actor-critic algorithms, each utilizing a distinct method to estimate the gradient and update policy parameters in the descent direction. In contrast, we consider a broader composition of a general coherent risk measure and a general loss function, allowing more flexible objectives. Note that although the composition of a coherent risk measure and a convex loss function is convex in the occupancy measure, it is generally non-convex in the policy parameters, which introduces additional challenges for our convergence analysis. The work most relevant to ours is perhaps \cite{zhang2020variational}, which addresses a reinforcement learning (RL) problem with a general convex loss function and derives the variational policy gradient theorem with a global convergence guarantee. However, our work differs in that we consider an offline planning problem in an MDP with unknown transition probabilities, which are estimated from data. Therefore, we address not only a general convex loss function but also epistemic uncertainty. This introduces additional challenges related to risk measures, and the robustness of the proposed formulation is a key consideration.

Our contributions are summarized as follows: (1) We propose a Bayesian risk formulation for MDPs with a general convex loss function and develop a policy gradient algorithm to solve for the optimal policy. The proposed formulation jointly mitigates both epistemic and intrinsic uncertainty; (2) We extend the envelope theorem to   the dual representation of the coherent risk measure, and then apply the envelope theorem to derive the policy gradient. Our extension from the discrete case to the continuous case for the envelope theorem may be of independent interest; (3) 
We prove the convergence of the proposed algorithm and establish its convergence rate as {
$\mathcal{O}(T^{-1/2}+r^{-1/2})$}, where $T$  is the number of policy gradient iterations and $r$ is the sample number of the gradient estimator; (4) We extend our policy gradient algorithm to the episodic setting, and prove the asymptotic convergence of the  episodic optimal policy of our Bayesian formulation to a globally optimal policy of the MDP problem under the true environment. Moreover, we show the number of iterations required in any episode to maintain an optimality gap { $\mathcal{O}(\epsilon)$} under our Bayesian formulation.

\section{Problem Formulation}\label{sec:formulation}
Consider an infinite-horizon Markov Decision Process (MDP) over a finite state space $\mathcal{S} $ and a finite action space $\mathcal{A}$. For each state $s\in \mathcal{S}$ and action $a\in \mathcal{A}$, a transition to the next state $s'$ follows the transition kernel $P^*$, i.e. $s'\sim P^*(\cdot|s,a) $. A stationary policy $\pi$ is defined as a function mapping from the state
space to a probability simplex $\Delta(\cdot)$ over the action space.
Given any transition probability $P$, define $\lambda^{\pi,P}$ to be the discounted state-action occupancy measure under policy $\pi$ :
\begin{equation}
\label{eq:def_occu}
  \lambda_{s a}^{\pi,P}=\sum_{t=0}^{\infty} \gamma^t \cdot \mathbb{P}\left(s_t=s, a_t=a \mid \pi, s_0 \sim \tau,P \right)
\end{equation}
for any $ (s,a) \in \mathcal{S} \times \mathcal{A}$, where $\tau$ is the initial distribution, $\gamma \in (0,1)$ is the discount factor.

As mentioned in introduction, in many application problems such as  a self-driving car in a dynamic urban environment, the decision maker faces two kinds of
challenges: the epistemic uncertainty about the environment and a general performance measure for  the intrinsic uncertainty. In this paper, we aim to address both challenges together. We consider a general loss function $F(\lambda,P)$ defined over the occupancy measure $\lambda$ and transition kernel $P$, which is assumed to be convex in $\lambda$. 
In practice, the true distribution $P^*$ is usually unknown and needs to be estimated. In this work, we take a Bayesian approach to estimate the environment. We assume that the transition kernel $P^* \equiv P_{\theta^*}$ is parameterized by $\theta^*$, where $\theta^* \in \Theta$ is the true but unknown parameter value, $\Theta \subseteq \mathbb{R}^{p}$ is the parameter space, and $p$ is the dimension of $\Theta$. Many real-world problems exhibit the characteristic of relying on a parametric assumption.
%, which  reveals the intrinsic feature dimension. 
In the example of a self-driving car, the noise in sensor measurements may be assumed to follow an unknown Gaussian distribution.

Under the parametric assumption,  we assume we have access to some data which are  state transitions $\zeta=(s,a,s')$, where $s'$ follows the distribution $P_{\theta^*}(\cdot|s,a)$ and define $P_{\theta^*}(\zeta):=P_{\theta^*}(s'|s,a)$. Now given a fixed batch of data $\zeta^{(N)}$ of $N$ samples, we can update the posterior distribution (denoted by $\mu_N$) on the parameter $\theta$ using the Bayes rule: $\mu_N(\theta)=\frac{P_{\theta}(\zeta^{(N)}))\mu_0(\theta)}{\int_{\theta'} P_{\theta'}(\zeta^{(N)})\mu_0(\theta')d\theta'}$, where $\mu_0$ is a prior distribution of $\theta$ we assume. 
Furthermore, as discussed before, model mis-specification caused by the lack of data could lead to sub-optimality of the learned policy when it is implemented in a real-world setting. Hence, we further impose a risk functional on the objective with respect to (w.r.t.) the Bayesian posterior to account for the epistemic uncertainty, which results in the following composed formulation:
\begin{equation}\label{original_policy_optimimzation}
    \min_{\pi} \rho_{\theta\sim\mu_N}(F(\lambda^{\pi,P_\theta},P_\theta))
\end{equation}
where $\rho$ is a general coherent risk measure 
\footnote{Let $(\Omega, \mathcal{F}, \mathbb{P})$w.r.t. the posterior $\mu_N$ be a probability space and $\mathcal{X}$ be a linear space of $\mathcal{F}$-measurable functions $X: \Omega \rightarrow \mathbb{R}$. A risk measure is a function $\rho: \mathcal{X} \to \mathbb{R}$ which assigns to a random variable $X$ a real number representing its risk. A coherent risk measure satisfies properties of monotonicity, sub-additivity, homogeneity, and translational invariance.} 
w.r.t. the posterior $\mu_N$.{ We aim to solve problem \eqref{original_policy_optimimzation} in this paper.} Detailed introduction about coherent risk measures can be  found in \citep{artzner1999coherent}. 
By this formulation, we look for a policy that minimizes a performance measure taking into account to the epistemic uncertainty caused by lack of data for a general convex loss function. %For safe learning, the inner layer $F$ is a Lagrangian function corresponding  to the random transitions of states given a fixed environment, and the outer layer $\rho$ corresponds to the measurement of estimating unknown environment. This formulation focuses on addressing both uncertainties together.
If $F$ is a linear function of $\lambda$, i.e. $F(\lambda, P)=\langle \lambda, c \rangle$ for a cost vector $c \in \mathbb{R}^{|\mathcal{S}|\times |\mathcal{A}|}$, and
the posterior $\mu_N$ is a singleton on the true parameter $\theta^*$, then the risk measure just considers the performance on this singleton and \eqref{original_policy_optimimzation} reduces to the classical MDP problem. 
% Next, we give some examples that are not in the classical form of MDP but fall into our framework. We list one example below, which is motivated by safe RL, and more examples can be found in Appendix~\ref{appendix:loss_function}.
In Appendix~\ref{appendix:loss_function}, we give two examples that are not in the classical form of MDP but fall into our framework.

\color{red}

\color{black}

% \textcolor{red}{Explain more on the motivation. Raise some special cases of it. Any  real-world use}
\section{Policy Gradient Algorithm: Derivation and Estimation}\label{sec:algorithm}
As discussed in the introduction, the dynamic programming type of algorithms may not be applicable to a general convex loss function $F(\cdot)$. So we adopt the policy gradient algorithm, which directly optimizes parameterized policies. Consider a stochastic parameterized policy $\pi_{\alpha}: \mathcal{S} \to \Delta(\mathcal{A})$, parameterized by $\alpha \in W \subset \mathbb{R}^{d}$. To directly work on the parameterized policy, we denote $F(\lambda^{\pi_\alpha,P_\theta},P_\theta)$ by $C(\alpha,\theta)$. The policy optimization problem \eqref{original_policy_optimimzation}  then becomes:
\begin{align}
\label{param_policy_optimimzation}
    \min_{\alpha}G(\alpha):= \rho_{\theta\sim\mu_N}(C(\alpha,\theta)).
\end{align}
It is worth mentioning that 
$G(\alpha)$ is not necessarily a convex function though $F$ is convex w.r.t. $\lambda$. So we can only reach a stationary point of $G(\alpha)$ by the policy gradient descent method (see more detailed discussion in Section~\ref{subsec:convergence_alg}). In the rest of the section, we derive the policy gradient to the proposed formulation \eqref{param_policy_optimimzation} using the envelope theorem, and construct the policy gradient estimator. It should be noted that our proposed formulation allows for flexible methods to estimate the policy gradient, including the variational approach such as in \citep{zhang2020variational}, and the zeroth-order method such as in \citep{balasubramanian2022zeroth}.

\subsection{Preliminaries}
% Note that $(\Theta, \mu_N)$ is a probability space.  
To ensure the objective $G(\alpha)$ is well defined, we first assume that  $C(\alpha,\theta) \in  \mathcal{Z}:=L_p(\Theta,\mu_N)$.
\begin{assumption}
\label{assum:C_lp}
    $C(\alpha,\theta)\in\mathcal{Z}=\{f:\|f\|_p:=\left(\int_\Theta |f(\theta)|^p d\mu_N(\theta)\right)^{1/p} <\infty\},\forall \alpha \in W$,  for some $p\ge 1$.
\end{assumption}
The choice of $p$ depends on the specific coherent risk measure. For example, $p$ can be chosen as $1$ for CVaR  introduced in Example \ref{examlple:risk-averse CMDP}.
Let $\mathcal{B}:=\{\xi \in \mathcal{Z}^*: \int_{\Theta} \xi(\theta)\mu_N(\theta)d\theta=1, \xi \succeq 0\},$ where $\mathcal{Z}^*:=L_q(\Theta,\mu_N)$ is the dual space of $\mathcal{Z}$ with $1/p+1/q=1.$
As shown in \citep{shapiro2021lectures}, a coherent risk measure has a well-known dual representation.

\begin{theorem}
\label{thm:dual_risk}{(Theorem 6.6 in \citep{shapiro2021lectures}.)}
A risk measure $\rho: \mathcal{Z} \rightarrow \mathbb{R}$ is coherent if and only if there exists a convex bounded and closed set (also known as risk envelope) $\mathcal{U}=\mathcal{U}(\mu_N) \subset \mathcal{B}$ such that 
$
\rho(Z)=\max _{\xi: \xi  \in \mathcal{U}\left(\mu_N\right)} \mathbb{E}_{\xi}[Z],
$
where $\mathbb{E}_{\xi}[Z] := \int_{\theta \in \Theta} Z(\theta ) \xi(\theta ) \mu_N(\theta)d\theta$ .
\end{theorem}
Note $\xi$ could be viewed as perturbation on the posterior $\mu_N$ that satisfies certain conditions, and the risk measure can be understood as the extreme performance for these perturbations. Theorem \ref{thm:dual_risk} implies that a functional $\rho$ defined by  $
\rho(Z)=\max _{\xi: \xi  \in \mathcal{U}} \mathbb{E}_{\xi}[Z]$ is a coherent risk measure if $\mathcal{U} \subset \mathcal{B}$ is convex, bounded and closed. In this paper we only focus on a  class of coherent risk measures $\rho$ following Definition \ref{assum: envelop} throughout the paper. 
\begin{definition}
\label{assum: envelop}
For each given policy parameter $\theta \in \mathbb{R}^K$, there exists an expression for the risk envelope $\mathcal{U}$ of the coherent risk measure  $\rho$ in the following form:
\begin{align*}
\begin{aligned}
  &\mathcal{U}\left(\mu_N\right)=\{\xi \in \mathcal{Z}^* : g_e\left(\xi,\mu_N\right)=0, \forall e \in \mathcal{E},f_i\left(\xi, \mu_N\right) \leq 0,\\
  &  \forall i \in \mathcal{I},\int_{\theta \in \Theta} \xi(\theta) \mu_N(\theta)d\theta=1, \xi(\theta) \geq 0,\|\xi\|_q \le B_q\},
\end{aligned}
\end{align*}
where $g_e\left(\xi, \mu_N\right)$ is an affine function in $\xi$, $f_i\left(\xi,  \mu_N\right)$ is a convex function in $\xi$, $\|\cdot\|_q $ is the $L_q$ norm in $\mathcal{Z}^*$, and there exists a strictly feasible point $\bar{\xi}$. $\mathcal{E}$ and $\mathcal{I}$ here denote the sets of equality and inequality constraints, respectively. Furthermore, for any given $\xi \in \mathcal{B}, f_i(\xi, \mu_N)$ and $g_e(\xi, \mu_N)$ are twice differentiable in $\mu_N$, and there exists a $M>0$ such that  $\forall \omega \in \Omega:$
\begin{align*}
    \max \left\{\max _{i \in \mathcal{I}}\left|\frac{d f_i(\xi, \mu_N)}{d \mu_N(\theta)}\right|, \max _{e \in \mathcal{E}}\left|\frac{d g_e(\xi, \mu_N)}{d \mu_N(\theta)}\right|\right\} \leq M.
\end{align*}
\end{definition}

The conditions on $g_e$ and $f_i$  ensure that risk envelope $\mathcal{U}(\mu_N)$ is a convex closed set, and the condition $\|\xi\|_q\le B_q$ makes $\mathcal{U}(\mu_N)$ bounded. A similar assumption is considered in Assumption 2.2\citep{tamar2015policy}. The assumption about bounded derivatives can be easily satisfied if $\Theta$ is compact. While \citep{tamar2015policy} only consider the case where $\Theta$ is finite,  we extend it to the continuous case, leading to a functional problem over an infinite dimensional space instead of a finite-dimensional case. Therefore, we extend the result in \citep{tamar2015policy} to the infinite dimensional space, which is shown in Theorem \ref{thm:gradient_coherent}.
Notably, the function forms of $g_e(\cdot)$ and $f_i(\cdot)$ can be exactly specified for a given coherent risk measure. We refer the readers to Appendix~\ref{appendix:example_envelop} and Section 6.3.2 \citep{shapiro2021lectures} for some examples of the envelope set for coherent risk measures, which cover most common coherent risk measures. { The detailed discussion about differences between our approach and that of \citet{tamar2015policy} is provided in Appendix \ref{appendix:difference_tamar}.}
% More examples that satisfy Definition \ref{assum: envelop} can be found in Section 6.3.2 \citep{shapiro2021lectures}, which covers most common coherent risk measures.

\subsection{Derivation of Policy Gradient}
According to Theorem \ref{thm:dual_risk}, we can write the coherent risk measure as a maximization problem \eqref{eq:dual_problem}, where the decision variable is $\xi$ and the objective is a linear functional of $\xi$, and define the Lagrangian function \eqref{eq:lagrangian} for problem \eqref{eq:dual_problem}:
\begin{equation}
\label{eq:dual_problem}
\begin{aligned}
    &\rho_{\theta\sim\mu_N}(C(\alpha,\theta))=\max _{\xi: \xi \in \mathcal{U}\left(\mu_N\right)} \mathbb{E}_{\xi}[C(\alpha,\theta)]
    = \max _{\xi: \xi \in \mathcal{U}\left(\mu_N\right)} \int_{\theta \in \Theta} \xi(\theta) \mu_N(\theta)C(\alpha,\theta) d\theta.
\end{aligned}
\end{equation}
\begin{equation}
\label{eq:lagrangian}
\begin{aligned}  &L_\alpha(\xi,\lambda^{\mathcal{P}},\lambda^{\mathcal{E}},\lambda^{\mathcal{I}})=\int_{\theta \in \Theta} \xi(\theta) \mu_N(\theta) C(\alpha,\theta)d\theta-\lambda^{\mathcal{P}}\left(\int_{\theta \in \Theta} \xi(\theta) \mu_N(\theta)d\theta-1\right)\\
&-\sum_{e \in \mathcal{E}} \lambda^{\mathcal{E}}(e) g_e\left(\xi, \mu_N\right)-\sum_{i \in \mathcal{I}} \lambda^{\mathcal{I}}(i) f_i\left(\xi, \mu_N\right).
\end{aligned} 
\end{equation}
Using the Lagrangian relaxation \eqref{eq:lagrangian}, we derive the policy gradient for \eqref{eq:dual_problem} in Theorem \ref{thm:gradient_coherent}. For this purpose, we need some mild assumptions  on the objective function. 

\begin{assumption}
\label{assum:Lipschitz_F_C}
(1) $\nabla_{\lambda} F(\lambda, P)$ is uniformly bounded by $L_{F,\infty}$ for any $\lambda$ and $P$ w.r.t. $\|\cdot\|_\infty$; (2) $\nabla_\alpha C(\alpha, \theta)$ is  $L_{\theta,2}$-Lipschitz continuous w.r.t. $\theta \in \Theta$ and $\|\cdot\|_2$ for any $\alpha \in W$; (3) $\nabla C(\alpha,\theta)$ is uniformly bounded by $B$ for any $\alpha \in W$ and $\theta \in \Theta$ w.r.t. $\|\cdot\|_2$; (4) $\Theta \subseteq \mathbb{R}^p$ is compact and convex; (5)  $W$, the domain of $\alpha$, is bounded by $B_W$.
\end{assumption}
Assumption~\ref{assum:Lipschitz_F_C} requires the uniform boundedness and Lipschitz continuity of $\nabla C$ and $\nabla F$, where  $C(\alpha,\theta)=F(\lambda^{\pi_\alpha,P_\theta},P_\theta)$. One sufficient condition  easy to verify for Assumption~\ref{assum:Lipschitz_F_C} to hold is: each component in the composed function $F(\lambda^{\pi_\alpha,P_\theta},P_\theta)$ is (somewhere) twice continuously differentiable w.r.t parameters $\alpha, \theta$, and the domains of two parameters are compact convex sets.  { Standard assumptions on the policy parameterization in RL are closely related to Assumption 3.2, and the detailed discussion and a continuous-control example  are offered in Appendix \ref{appendix:Lipshitz}. }

\begin{theorem}
\label{thm:gradient_coherent}
Assume that Assumption \ref{assum:C_lp} and \ref{assum:Lipschitz_F_C} hold, and $\rho$ satisfies Definition \ref{assum: envelop}. Assume that $\mu_N$ is a Radon measure (see Appendix \ref{appendix: radon} for definition). Define $\xi^*\in\arg \max_{\xi\ge 0, \|\xi\|_q\le B_q} \min _{\lambda^{\mathcal{I}}\geq 0,\lambda^{\mathcal{P}},  \lambda^{\varepsilon}} L_\alpha\left(\xi, \lambda^{\mathcal{P}}, \lambda^{\mathcal{E}}, \lambda^{\mathcal{I}}\right).$
Then we have the policy gradient
\begin{equation}
\label{eq:risk_measure}
\begin{split}
      g(\alpha)&:=  \nabla_\alpha \rho_{\theta\sim\mu_N}(C(\alpha,\theta))=\int_{\theta \in \Theta} \xi_\alpha^*(\theta) \mu_N(\theta)\nabla_\alpha C(\alpha,\theta) d\theta.   
\end{split}
\end{equation}
\end{theorem}
% \begin{proof}
% We give a proof sketch and refer the readers to Appendix~\ref{appendix:gradient_coherent} for the proof details. The proof consists of two parts. The first part is to show \eqref{eq:lagrangian} satisfies the Slater's condition and thus strong duality holds. The second part is to apply the  envelope theorem  to \eqref{eq:lagrangian}. It should be noted that $\xi$ is in an infinite-dimensional function space instead of a finite-dimensional vector space, which leads to a harder functional problem over an infinite dimensional space instead of an easier finite-dimensional case. Here we modify the original assumption in \citep{tamar2015policy} to make the envelope set separable, ensure that the conditions of envelope theorem  hold, and then extend their result to the functional version shown in this theorem.
% \end{proof}
Proof details of Theorem \ref{thm:gradient_coherent} can be found in Appendix~\ref{appendix:gradient_coherent}.
Theorem \ref{thm:gradient_coherent} implies that we can plug in a saddle point of Lagrangian \eqref{eq:lagrangian} into \eqref{eq:risk_measure} to get the policy gradient. However, \eqref{eq:risk_measure} still involves $\nabla C$, the gradient of the loss function, and the integration w.r.t. the posterior $\mu_N$. In the next subsection, we show how to estimate the policy gradient in \eqref{eq:risk_measure}.

\subsection{Construction of the Policy Gradient Estimator}
In this section, we focus on how to estimate the policy gradient $g(\alpha)$  and  denote its estimator by $\widehat{g}(\alpha)$.
We first need to find $\xi^*$ in Theorem \ref{thm:gradient_coherent}. 
For some coherent risk measures, the  closed-form expression of $\xi^*$ is known. For CVaR with risk level $\beta \in (0,1)$, $\xi^*(\theta)=\frac{1}{1-\beta}$ if $C(\alpha,\theta)\ge v_{\beta}$ and 0 otherwise, where $v_\beta$ is the $\beta$-quantile of $C(\alpha,\theta)$. For a general coherent risk measure, we can use the approach sample average approximation (SAA). We first sample $\theta_k, k =1,\dots,r, $ from $\mu_N$, and then solve the following SAA problem for the solution $\xi^*(\theta_k)$ for each $k$:
\begin{equation}
\label{eq:max_min_SAA}
\begin{split}
   &\max_{\substack{\xi \ge 0,\\(\sum_{k=1}^{r}|\xi(\theta_k)|^q)/r\le B_q}} \min _{\lambda^{\mathcal{I}} \geq 0,\lambda^{\mathcal{P}}, \lambda^{\mathcal{E}} }
   \frac{1}{r}\sum_{k=1}^{r}\xi(\theta_k)  C(\alpha,\theta_k)-\lambda^{\mathcal{P}}\left( \frac{1}{r}\sum_{k=1}^{r}\xi(\theta_k ) -1\right)\\&-\sum_{e \in \mathcal{E}} \lambda^{\mathcal{E}}(e) g_e\left(\xi^{(r)}, \mu_N{(r)}\right)-\sum_{id \in \mathcal{I}} \lambda^{\mathcal{I}}(k) f_i\left(\xi^{(r)}, \mu_N{(r)}\right)  
\end{split}
\end{equation}
Notice the objective in \eqref{eq:max_min_SAA} is linear w.r.t.  $\lambda^{\mathcal{I}},\lambda^{\mathcal{E}}$ and concave w.r.t $\xi$, and the domain of $\xi$ is a convex bounded set in $\mathbb{R}^{r}$. Thus, \eqref{eq:max_min_SAA} can be solved  by any max-min optimization algorithm for a concave-convex function, such as alternating gradient descent ascent. Here we assume that we can solve \eqref{eq:max_min_SAA} to derive $\xi^*(\theta_k)$ accurately for each $k$.
% Any multivariate interpolation method can be used
% to estimate the value of $\xi^*$ on the whole $\Theta$. 
Apart from $\xi^*$, we need to estimate  $\nabla_\alpha C(\alpha,\theta)$ and the integral $\int_{\theta \in \Theta} \xi_\alpha^*(\theta) \mu_N(\theta)\nabla_\alpha C(\alpha,\theta)$. To estimate $\nabla_\alpha C(\alpha,\theta)$, any plug-in estimation method satisfies our demand. Here, we adopt the variational policy gradient theorem in \citep{zhang2020variational}, which considers the policy gradient for a concave function defined on the occupancy measure for a RL problem. Different from our Bayesian-risk problem with a general loss function, \citep{zhang2020variational} only considers the inner-layer  $F$ of our objective \eqref{original_policy_optimimzation} in the online setting. It should also be noticed that their method can be replaced by other methods such as the zeroth-order estimation method in \citep{balasubramanian2022zeroth}. While the variational policy gradient theorem require access to the conjugate function $F^*$, which may be difficult to calculate in some cases, zeroth-order method only requires function evaluation of $F$ but leads to higher computational cost in general cases. 

\begin{lemma}
    {(Theorem 3.1 in \citep{zhang2020variational})} 
\label{thm:nablaC}Suppose $F$ is convex  and continuously differentiable in an open neighborhood of $\lambda^{\pi_\alpha,P_\theta}$. Fix the transition kernel $P_\theta$ and  denote $V(\alpha ; z)$ to be the expected cumulative cost of policy $\pi_\alpha$ when the cost function is $z$, and assume $\nabla_\alpha V(\alpha ; z)$ always exists. Then we have
\begin{equation}
 \label{eq:min_max_nablaC}
 \begin{aligned}
     &\nabla_\alpha C\left(\alpha,\theta\right)=-\lim _{\delta \rightarrow 0_{+}} \underset{x\in \mathbb{R}^{SA}}{\operatorname{argmin}} \sup_{z\in \mathbb{R}^{SA}}\{ V(\alpha ; z)+\delta \nabla_\alpha V(\alpha ; z)^{\top} x-F^*(z)+\frac{\delta}{2}\|x\|^2\},  
 \end{aligned}
 \end{equation}
% \vspace{-0.1cm}
% {\small
% \begin{align} \label{eq:min_max_nablaC}
%   \nabla_\alpha C\left(\alpha,\theta\right)=-\lim _{\delta \rightarrow 0_{+}} \underset{x\in \mathbb{R}^{SA}}{\operatorname{argmin}} \sup_{z\in \mathbb{R}^{SA}}\left\{ V(\alpha ; z)+\delta \nabla_\alpha V(\alpha ; z)^{\top} x-F^*(z)+\frac{\delta}{2}\|x\|^2\right\}
% \end{align}
% }
% \vspace{-0.3cm}
where 
$\ V(\alpha ; z)  =\langle z, \lambda(\alpha,\theta)\rangle$, $\nabla_\alpha V(\alpha ; z)^{\top} x =\langle z, \nabla_\alpha \lambda(\alpha,\theta) x\rangle, F^*(z):= \sup_{x\in \mathbb{R}^{SA}} x^\top z-F(x) $ is the Fenchel conjugate of $F$.
\end{lemma}

% A natural idea to calculate $\nabla_\alpha C$ is to use chain rule, i.e. $\nabla_\alpha C=\nabla_\lambda F \cdot \nabla_\alpha \lambda$. 
% However, 
It may have a high computational cost if we directly estimate each part at a specific $\alpha$ in $\nabla_\alpha C=\nabla_\lambda F \cdot \nabla_\alpha \lambda$. The variational policy gradient method bypasses this issue  by changing this problem into a problem of calculating some linear functions and the conjugate function at any $z$, shown in \eqref{eq:min_max_nablaC}. 
% To solve \eqref{eq:min_max_nablaC}, we need to estimate $V(\alpha;z)$ and $\nabla_\alpha V(\alpha;z)$. 
\citep{zhang2020variational} considers an online setting and thus they need to interact with the environment to estimate $\nabla_\alpha C$. In our offline setting, we can directly solve \eqref{eq:min_max_nablaC} to get $\nabla_\alpha C$. 
An example algorithm to solve \eqref{eq:min_max_nablaC} is given in Appendix~\ref{appendix:nablaC_solve}.  
To evaluate the integral $\int_{\theta \in \Theta} \xi_\alpha^*(\theta) \mu_N(\theta)\nabla_\alpha C(\alpha,\theta)d\theta$, we  use  samples $\theta_k$ to construct the policy gradient estimator
\begin{equation}
\label{eq:gradient_estimator_formulation}
    \nabla_\alpha \rho_{\theta\sim\mu_N}(C(\alpha,\theta))\approx 
    \widehat{g}(\alpha):=\frac{1}{r}\sum_{k=1}^{r}\xi^*(\theta_k)\nabla_\alpha C\left(\alpha,\theta_k\right).
\end{equation}
In this paper, we assume the access to samples from the posterior distribution $\mu_N$. 
While computing the posterior often requires costly methods like Markov Chain Monte Carlo (MCMC), utilizing a conjugate prior yields closed-form updates  for the posterior parameters, improving efficiency. 
Computing the posterior typically requires expensive methods such as Markov Chain Monte Carlo (MCMC).
However, by utilizing a conjugate prior, we obtain a closed-form expression for the posterior parameters, making the calculation more efficient. 
% Bayesian update can also be done by neural network by normalizing the output of neural network into a probability. Notably, we resort to solving the SAA problem \eqref{eq:max_min_SAA} only when we cannot derive the closed-form expression for $\xi^*$, which depends on the risk measure we choose.
\subsection{Full Algorithm and Episodic Setting}
\label{subsec:Full Algorithm}
To perform policy gradient optimization, we iteratively use the gradient descent step \eqref{eq:gradient_descent_step}, where $\eta_t$ is the step size, and $\operatorname{Proj}_{W}(x)= \operatorname{\arg\min}_{y\in W}\|y-x\|_2^2$ is the projection  into the parameter space $W$. We summarize the full algorithm in Algorithm~\ref{alg:sgd_coherent} in Appendix \ref{appendix:algorithms}.
\begin{equation}
\label{eq:gradient_descent_step}
\begin{split}
\alpha_{t+1}&=\arg\min\limits_{\alpha \in W}\langle \widehat{g}(\alpha_t),\alpha-\alpha_t \rangle+\frac{\eta_t}{2}\|\alpha-\alpha_t\|^2 =\operatorname{Proj}_{W}\left(\alpha_t-\frac{1}{\eta_t}\widehat{g}(\alpha_t)\right).
\end{split}
\end{equation}
So far we have considered the offline setting with a fixed batch of data, but in many application problems data can be collected periodically. Again, consider a self-driving car as an example: the car is trained in an offline setting and then deployed to a real environment for a test drive while collecting more data from the environment. The collected data can be then used to learn about the environment and update the policy. This process can be repeated iteratively. 
%Consider a recommendation system where the data collection process alternates between making recommendations and learning about customers' preferences based on their feedback. In this scenario, data arrives incrementally, prompting the system to iteratively learn about the environment, update the policy, and deploy the updated policy to collect new data. 
Thus, we extend  our approach to an episodic setting as described above. A potential approach is to use Algorithm \ref{alg:sgd_coherent} to make policy updates during each episode, as detailed in Algorithm \ref{alg:sgd_episodic} in Appendix \ref{appendix:algorithms}.

\section{Convergence Analysis}\label{sec:convergence}

In this section, we analyze the convergence properties of Algorithm~\ref{alg:sgd_coherent} and Algorithm~\ref{alg:sgd_episodic}. We begin by establishing the error bound for the policy gradient estimator. Next, we demonstrate the finite-time first-order convergence rate is {
$\mathcal{O}(T^{-1/2}+r^{-1/2})$, where $T$  is the number of policy gradient iterations and $r$ is the sample number of the gradient estimator.} Furthermore, we prove the consistency of the proposed Bayesian risk formulation, meaning that the optimal policy obtained through this formulation converges to the one obtained by solving the true problem as the number of initial data points $N$ approaches infinity. Lastly, for the episodic setting we show the number of iterations required in any episode to maintain an {$\mathcal{O}(\epsilon)$}-error bound over all episodes, which implies the global convergence of Algorithm \ref{alg:sgd_episodic} as $N$ goes to infinity.

\subsection{Estimation Error of the Policy Gradient}
\label{subsec:Estimation Error}
\begin{assumption}
\label{assum:xi_variance} 
Assume  $\xi^*$ in Theorem~\ref{thm:gradient_coherent} satisfies $\sup_{\alpha\in W} \mathrm{Var}_{\theta\sim\mu_N}[\xi^*(\theta)\nabla C(\alpha,\theta)]=\sigma_{\xi}<\infty$.
\end{assumption}
Assumption \ref{assum:xi_variance} requires the uniformly bounded variance of $\xi^{*}\nabla C$. It is hard to show some property of $\xi^*$ in a  general case as the envelope set is given in a general form.  One sufficient condition for Assumption 4.1 to hold is that $\xi^*$ is bounded on  $\Theta$. As  $\Theta$ is a compact and convex set, it is not a strong condition.
\begin{theorem}
\label{thm:estimator_error_coherent}
Assume  Assumption~\ref{assum:C_lp}, \ref{assum:Lipschitz_F_C} and \ref{assum:xi_variance} hold. By using $r$  samples for gradient estimator in \eqref{eq:gradient_estimator_formulation}, the gradient estimation error is
$
    \label{eq:gradient_error}
     \mathbb{E}\left[\|\widehat{g}(\alpha)-g(\alpha)\|_2^2\right]\le\frac{\sigma_\xi}{r}, \forall\alpha \in W.
$
\end{theorem}
Theorem~\ref{thm:estimator_error_coherent} implies that the sample complexity of $\Theta(1/\epsilon)$ is required to achieve the  estimation error  $\mathcal{O}(\epsilon)$. Please refer to Appendix \ref{appendix:estimator_error_coherent} for the detailed proof.

\subsection{Convergence Analysis}
\label{subsec:convergence_alg}
% Recall that $G(\alpha):=\rho_{\theta\sim\mu_N}(C(\alpha,\theta))$ is our objective. 

First we make an assumption about the Lipschitz continuity of $g(\alpha)$ in Assumption~\ref{assum:bounded_gradient}.

% \begin{assumption}{(Assumption 4.1 in \citep{zhang2020variational})}
% \label{assum:lambda}
% ~
% \begin{enumerate}
%     \item[(1)] $\lambda(\cdot)$ forms a bijection between $W$ and $\lambda(W)$, where $W$ and $\lambda(W)$ are closed and convex.
%     \item[(2)] The Jacobian matrix $\nabla \lambda(\alpha)$ is Lipschitz continuous in $\alpha$.
%     \item[(3)] There exists $\ell_\alpha>0$ s.t. $\| \lambda^{-1}(\lambda_1)-\lambda^{-1}(\lambda_2) \|_2 \leq \ell_\alpha\left\|\lambda_1-\lambda_2\right\|_2$ and for all $\lambda_1, \lambda_2 \in \lambda(W)$, where $\lambda^{-1}(\cdot)$ is the inverse mapping of $\lambda(\cdot)$.
% \end{enumerate}
% \end{assumption}

% Assumption 4.6 (3) means that if $\lambda, \lambda^{\prime}$ are close, $\alpha$ and $\alpha'$ are still close.

% \begin{theorem}{(Optimality gap)}
% \label{thm:converge_alg}
% Suppose that Assumption~\ref{assum:C_lp}, \ref{assum:Lipschitz_F_C}, \ref{assum:xi_variance},  \ref{assum:bounded_gradient} and \ref{assum:lambda} hold, and $\rho$ satisfies Definition \ref{assum: envelop}.  
% By choosing $\eta_t=L_G$ in Algorithm \ref{alg:sgd_coherent}, it holds that
% \[
% \mathbb{E}G(\alpha_{t+1})-G^*\le \mathcal{O}(1/t)+\mathcal{O}\left({r}^{-1/4}\right)
% \]
% where $G^*$ is the globally minimal value of $G$.
% % and $R^2=      \mathcal{O}\left(dn^{-1}+\epsilon_\lambda+\frac{d\epsilon_\lambda^2}{\nu^2}+\frac{d+\nu^2d^3}{m}\right)
% % $ is the bound for $\mathbb{E}\|[g-\widehat{g}]\|_2^2] $ in Theorem \ref{thm:error_estimator}.
% \end{theorem}

\begin{theorem}{(Stationary convergence)}
\label{thm:converge_alg}
Suppose that Assumption~\ref{assum:C_lp}, \ref{assum:Lipschitz_F_C}, \ref{assum:xi_variance} and  \ref{assum:bounded_gradient} hold.
%$ \forall \epsilon <\bar{\epsilon}$ with $\bar{\epsilon}$ defined in Assumption \ref{assum:local_invertible_parametrization}.
By choosing $\eta_t=2L_G$ in Algorithm \ref{alg:sgd_coherent}, it holds that $
    \mathbb{E}\|\nabla G(\alpha_{\text{out}})\|\le \mathcal{O}\left(T^{-1/2}+r^{-1/2}\right)$.
% and $R^2=      \mathcal{O}\left(dn^{-1}+\epsilon_\lambda+\frac{d\epsilon_\lambda^2}{\nu^2}+\frac{d+\nu^2d^3}{m}\right)
% $ is the bound for $\mathbb{E}\|[g-\widehat{g}]\|_2^2] $ in Theorem \ref{thm:error_estimator}.
\end{theorem}

% The first half of the proof is based on the proof of Theorem 4.4 \citep{zhang2020variational}.
Theorem \ref{thm:converge_alg} shows that the gradient bound of the output policy consists of two parts: an asymptotically diminishing error bound $T^{-1/2}$ in the exact setting and an estimation error bound $r^{-1/2}$ of the policy gradient.{ 
% The samples are from the posterior $\mu_N$ and 
The total
sample complexity from the posterior $\mu_N$ to achieve accuracy $\mathcal{O}(\epsilon)$ is $\mathcal{O}(\epsilon^{-4})$ by choosing $T=\epsilon^{-2}$ and $r=\epsilon^{-2}$}. The proof and assumptions are shown in Appendix \ref{appendix:convergenc_alg}. Because of the intrinsic non-convex structure under a fixed posterior, only convergence to a stationary point can be achieved under a fixed posterior (or in other words, in a fixed episode). Detailed discussion about the non-convex structure is provided in Appendix \ref{appendix:non-convex}. However, global convergence can be achieved in the episodic setting as the posterior updates and converges. This is shown in Corollary \ref{thm:global convergence} later.

% \begin{theorem}{(Consistency)}
% \label{thm:consistency}
% Suppose that Assumption~\ref{assum:C_lp},  \ref{assum:Lipschitz_F_C}, \ref{assum:xi_variance},  \ref{assum:bounded_gradient}, \ref{assum:lambda} and  \ref{assum:bdro} hold, and $\rho$ satisfies Definition \ref{assum: envelop}.
% Then we have  $\sup_\alpha |\rho_{\theta\sim \mu_i}(C(\alpha,\theta))-C(\alpha,\theta^*)|$ tends to $0$ with probability 1 as $i\to \infty$, where the probability is  w.r.t.  infinite product probability measure of the data sequence.
% Moreover,
% $C(\alpha_N^*,\theta^*)- C(\alpha^*,\theta^*) \to 0$ with probability 1 as $N\to \infty$ , where $\alpha^*_N$ is a global minimizer of  $\rho_{\theta\sim\mu_N}(C(\alpha,\theta))$ and $\alpha^*$ is a global minimizer of $C(\alpha,\theta^*)$.
% \end{theorem}

\begin{theorem}{(Consistency of episodic optimal policy)}
\label{thm:consistency}
Suppose that Assumption~\ref{assum:C_lp},  \ref{assum:Lipschitz_F_C}, \ref{assum:xi_variance},  \ref{assum:bounded_gradient} and  \ref{assum:bdro} hold. Define $G_i(\alpha):=\rho_{\theta\sim \mu_i}(C(\alpha,\theta))$, which is the objective for the posterior $\mu_i$ with data size $i$.
Then we have  $D_i:=\sup_{\alpha\in W} |G_i(\alpha)-C(\alpha,\theta^*)|$ and $    E_i:=\sup_{\alpha\in W}\left\|\nabla_\alpha G_i(\alpha)-\nabla_\alpha C(\alpha,\theta^*)\right\|_2$ tend to $0$ with probability 1 as $i\to \infty$, where the probability is  w.r.t.  the data-generating distribution.
Moreover,
$C(\alpha_i^*,\theta^*)- C(\alpha^*,\theta^*) \to 0$ with probability 1 as $i\to \infty$ , where $\alpha^*_i$ is a global minimizer of  $G_i(\alpha)$ and $\alpha^*$ is a global minimizer of $C(\alpha,\theta^*)$.
\end{theorem}
\color{black}
As the  data size $N$ increases, the posterior distribution converges to a Dirac measure, which is a point mass at the true parameter $\theta^*$.  Consequently, the performance of the optimal policy for the posterior $\mu_N$  converges to the optimal policy under the true environment $\theta^*$, as demonstrated in Theorem \ref{thm:consistency}. 
Additional assumptions are required to ensure the convergence of a series of posteriors. Broadly speaking, it is necessary that all parameters and all data points have positive probabilities of being sampled under both the prior and posterior distributions, and that the interchangeability of limits and integrals is satisfied. Detailed proof and assumptions for Theorem \ref{thm:consistency} are provided in Appendix~\ref{appendix:consistency_proof}.
In the episodic setting, we iteratively use the current policy for data collection and posterior updates, and perform policy updates  based on the updated posterior, as described in Algorithm \ref{alg:sgd_episodic}.
A natural question arises: how many iterations are required within a given episode to achieve a certain level of accuracy? This is addressed in Theorem \ref{thm:two_episode}. 
Notably, the inner map $\lambda(\alpha,\theta^*)$ from policy parameter to occupancy measure is not necessarily convex in $\alpha$, though $ F (\lambda,\theta^*)$ is convex in $\lambda$. However, the hidden convexity can be utilized to get the global optimality under the true environment, regardless of the gradient estimation method. By utilizing the local bijection assumption of $\lambda(\cdot,\theta^*)$, a  stationary point is still globally optimal, {shown by Theorem 5.13 in \citep{zhang2021convergence}, which requires Assumption~\ref{assum:local_invertible_parametrization}. 
Assumption~\ref{assum:local_invertible_parametrization} can be satisfied when  $\lambda$ is a locally differentiable bijection on a compact convex set $W$.}

\begin{theorem}{(Finite-episode error bound)}
\label{thm:two_episode}
Suppose that Assumption~\ref{assum:C_lp}, \ref{assum:Lipschitz_F_C}, \ref{assum:xi_variance},  \ref{assum:bounded_gradient}, \ref{assum:bdro} and \ref{assum:local_invertible_parametrization} 
hold. 
Assume that $G_i(\alpha)$ defined in Theorem \ref{thm:consistency} has $L_{G,i}$-Lipschitz continuous gradient.
Let $\{\alpha_{i,j}\}_{i=1}^{N} \ _{j=0}^{t_i}$  be the generated policy parameter sequences for $N$ episodes by Algorithm \ref{alg:sgd_episodic}. 
For any $\epsilon>0$, if we choose $t_{i}=\Theta(L_{G,i}(E_{i-1}+D_{i})\epsilon^{-2})$, $r=\Theta(\epsilon^{-2})$, then we can keep a constant gradient bound
$\mathbb{E}\left[ \left(\sum_{j=0}^{t_{i}-1}\|\nabla G_{i}(\alpha_{i,j})\|_2\right)/{t_{i}}\right]\le\epsilon $
for each $i$. Furthermore, $\mathbb{E}C({\alpha_{\text{out}}},\theta^*)-C(\alpha^*,\theta^*)\le \mathcal{O}(\epsilon+E_N),$ where  $
D_i,E_i
$  are defined in Theorem \ref{thm:consistency}.

% For any $0<\epsilon$, if we want to keep
% a constant error bound $
% \epsilon$ for $\mathbb{E}[G_{i}(\alpha_{i,t_{i}})-G_{i}(\alpha_{i}^*)], i=1,\cdots,N$, then we need the sample number to be $r_i=\Theta(\epsilon^{-2}/L_{G,i})$  and $t_i$ to be at most
% $ \mathcal{O}(\epsilon^{-1}\log(\frac{D_i+D_{i-1}}{\epsilon}))$,
% where  $
% D_i:=\sup_\alpha |\rho_{\theta\sim \mu_i}(C(\alpha,\theta))-C(\alpha,\theta^*)|
% $  converges to $0$ when $i\to \infty$ as shown in Theorem \ref{thm:consistency}.  

\end{theorem}
\begin{corollary}{(Global convergence)}
\label{thm:global convergence}
    Using the same assumptions and notations in Theorem \ref{thm:two_episode}, $\mathbb{E}C({\alpha_{\text{out}}},\theta^*)-C(\alpha^*,\theta^*)\le \mathcal{O}(\epsilon)$ for any $\epsilon>0$ as $N \rightarrow \infty$.
\end{corollary}

Theorem \ref{thm:two_episode} offers theoretical advice on how to choose the iteration number in each episode. 
% When $i$ is small, 
% we need more iterations to make the Bayesian gradient small. 
Generally speaking,  we need fewer iterations to keep the gradient bound when $i$ grows since $D_i, E_i$ approaches $0$. Corollary \ref{thm:global convergence} is a direct result of Theorem \ref{thm:consistency} and  Theorem \ref{thm:two_episode}, and Corollary \ref{thm:global convergence} implies global convergence to the true optimal policy in the episodic setting (when the data size N increases to infinity).
Detailed proof can be found in Appendix~\ref{appendix:proof_of_two_episodes}.
\color{black}
\section{Numerical Experiments}
\label{sec:numerical}
We evaluate our proposed formulation and algorithm on the offline Frozen Lake problem~\citep{ravichandiran2018hands}, an OpenAI benchmark. We refer readers to Appendix~\ref{appendix:implementation}  for a detailed description of the problem. 
% We demonstrate  the performance of our proposed formulation and algorithm using an offline planning problem known as the frozen lake problem  \citep{ravichandiran2018hands}, an Open AI benchmark. For a detailed description of the problem, we refer readers to Appendix~\ref{appendix:implementation}. 
We consider different convex loss functions, including the mean and Kullback-Leibler (KL) divergence, for various tasks.
We compare the Bayesian Risk Policy Gradient (BR-PG) algorithm with CVaR risk measure under different risk levels $\beta=0,0.5,0.9$, respectively, with two other methods. 
The first is the empirical approach, which fits a maximum likelihood estimator (MLE) to the data and solves the MDP using the estimated parameters. The second is a modified offline version of distributionally robust Q-learning (DRQL)~\citep{liu2022distributionally}, which uses Q-learning to optimize worst-case performance over a KL divergence ball centered at the MLE kernel.
% The first is the empirical  approach, which computes a maximum likelihood estimator (MLE) for the parameters  from the given dataset and obtains a policy  by solving the MDP with the plugged-in MLE parameter value. 
% The second method is an offline version of distributionally robust Q-learning (DRQL) algorithm~\citep{liu2022distributionally}, which uses Q-learning to find the best policy in the worst-case distributional perturbation of the environment. \citep{liu2022distributionally} adopt a KL divergence ball centered at the true transition kernel as the environment's perturbation. 
When the risk level $\beta$ approaches $1$, Bayesian-risk performance is similar as the worst-case performance.
Since we are considering an offline planning problem, we modify the DRQL to interact with an offline simulator that uses the transition kernel with the MLE parameters derived from the data. 
% In other words, DRQL minimizes the worst-case performance for a KL divergence ball centered at the MLE kernel. 
For a fair comparison, we conduct DRQL experiments with different radii of the KL divergence ball. 
\textbf{Linear Loss.}  We consider the linear loss function, which corresponds to the total discounted cost in a classical MDP problem.
This is referred to as one replication, and we repeat  for 50 replications using different independent data sets.
% More  details can be found in Appendix \ref{appendix:implementation}. 
\textbf{Episodic Case.} We consider the episodic setting with 50 replications where the data collection and policy update are alternatively conducted.
% More implement details can be found in Appendix \ref{appendix:implementation} for 50 replications. 
\textbf{Mimicking a policy.}  Here we consider a different problem of mimicking an expert policy still using  Frozen Lake environment and 50 replications. The loss function  to minimize is defined as the KL divergence between state occupancy measure under the current policy and the expert state distribution. More implement details for three tasks can be found in Appendix \ref{appendix:implementation}.

\begin{table}[!ht]
\centering
% \captionsetup{font=small}
\caption{Results for frozen lake problem. Linear loss and positive-sided variance  at different risk levels $\alpha$ are reported for different algorithms and different data sizes with linear loss function. Standard errors are reported in parentheses. Escape probability $\theta_e=0.02$ and number of data points is $N=5$ and $50$.}
\resizebox{8.5cm}{!}{% 
\begin{tabular}{ccccc}
\hline
\multirow{2}{*}{Approach} & \multicolumn{2}{c}{N=5}             
 &\multicolumn{2}{c}{N=50}          
\\ \cline{2-5} 
                          & linear loss & positive-sided variance  & 
                          linear loss & positive-sided variance 
                          \\ \hline
BR-PG ($\beta=0$)         & 33.886(0.347 )     & 5.212        
 & 32.784
(0.00825)     & 0.0026    
\\ \hline
    BR-PG ($\beta=0.5$)       & 33.104
(0.127)     & 0.710                   
& 32.757
(0.00516)     &  0.00119                             
\\ \hline
BR-PG ($\beta=0.9$)       & {32.854
(0.0641)}     & {0.193
}             
 & {32.741
(0.00283)}     & {0.000376}                   
\\ \hline
Empirical Method               & 37.057(0.927)     & 34.387                     & 33.340
(0.0936)     &0.380         
\\ \hline
% DRQL                   & 32.815(0.422)     & 0.170      & 32.815(0.422)     & 0.170        
% \\ \hline
DRQL(radius=0.05)                  & 37.936(0.887)     & 26.554      & 34.365(0.366)     & 5.139      \\ \hline
DRQL(radius=1)                  & 35.216(0.732)     & 22.213      & 32.924(0.105)     & 0.519       \\ \hline
DRQL(radius=20)                  & 36.255(0.813)     &  24.622     & 32.855(0.063)     & 0.179       \\ \hline

{Optimal Policy under True Model}               & {32.499  }   &     & {32.499  }      &     \\ \hline
\end{tabular}
}
\label{table: frozen lake small,theta0.02}
\end{table}

\begin{figure}[ht]
\centering
    \begin{subfigure}[h]{0.23\textwidth}
         \centering
         \includegraphics[width=\textwidth, height=0.8\textwidth]{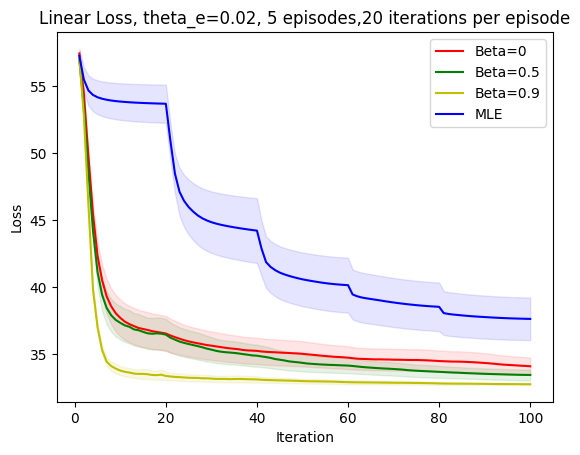}
         \caption{$5\times 20$ total iterations }
         \label{fig:Linear_eps_theta002_5_20}
    \end{subfigure}
     % \hfill
    \begin{subfigure}[h]{0.23\textwidth}
         \centering
         \includegraphics[width=\textwidth, height=0.8\textwidth]{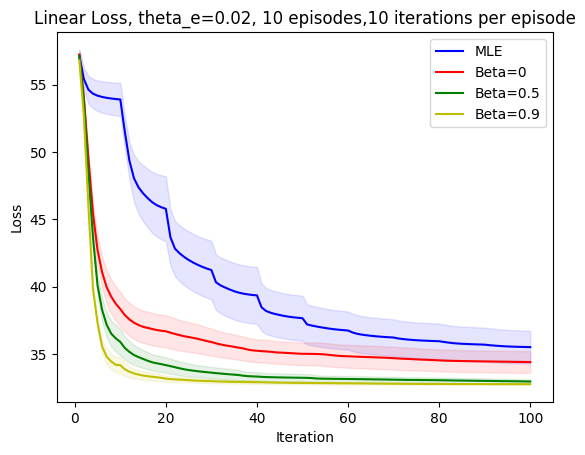}
         \caption{$10\times 10$ total iterations }
         \label{fig:Linear_eps_theta002_10_10}
    \end{subfigure}
    % \hfill
    \begin{subfigure}[h]{0.23\textwidth}
         \centering
         \includegraphics[width=\textwidth, height=0.8\textwidth]{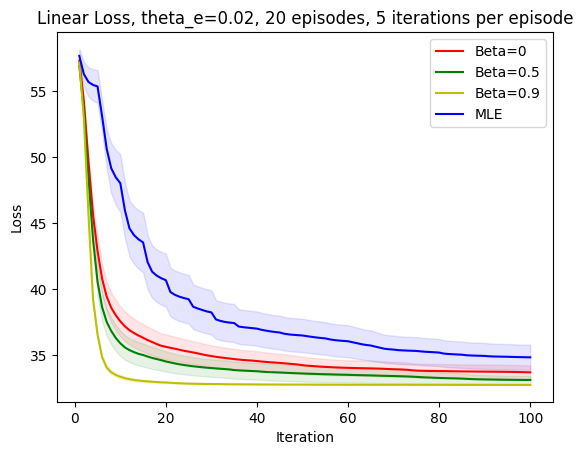}
         \caption{$20\times 5$ total iterations}
         \label{fig:Linear_eps_theta002_20_5}
    \end{subfigure}
    \caption{Results for episodic case with different episode numbers and iterations per episode  under the same escape probability $\theta_e=0.02$ and $50$ replications. Here the loss function is still chosen to be the linear loss. $95\%$ confidence intervals are reported by the shaded bands.}
    \label{fig:eps_linear_theta002}
\end{figure}

\textbf{Conclusions.} In each replication, data points are randomly sampled from the true distribution. While facing the epistemic uncertainty, BR-PG algorithm  provides robustness across different loss functions. Table \ref{table: frozen lake small,theta0.02} shows  that  our BR-PG algorithm has lower linear loss, standard error and positive-sided variance (psv), demonstrating more robustness in the sense of balancing the mean and variability of the actual cost. In contrast, the empirical approach performs badly when the data size is small, e.g. $N=5$, indicating that it is not robust against the epistemic uncertainty and suffers from the scarcity of data. DRQL also performs better than empirical method but worse than our algorithm in the sense of having larger mean and variance of the loss. 
\begin{wrapfigure}{r}{0.5\textwidth}
    \begin{subfigure}[b]{0.22\textwidth}
        \centering
         \includegraphics[width=\textwidth]{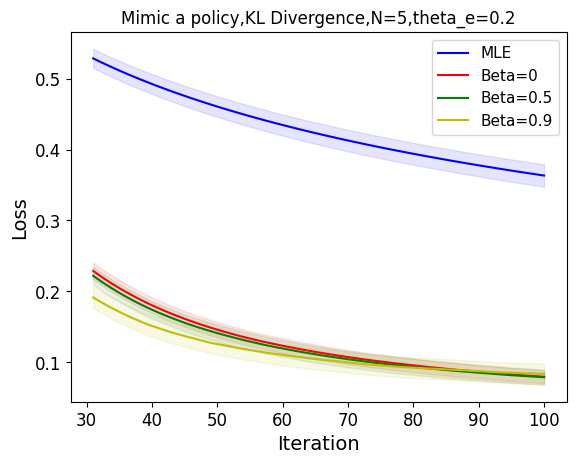}
         \caption{$N=5$}
         \label{fig:KLN5}
    \end{subfigure}
    \begin{subfigure}[b]{0.22\textwidth}
         \centering
         \includegraphics[width=\textwidth]{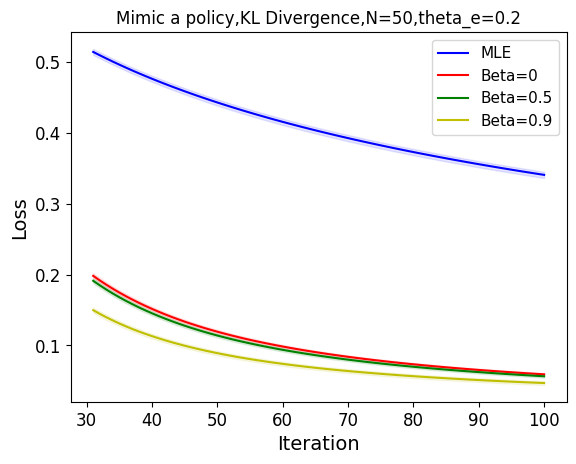}
         \caption{$N=50$}
         \label{fig:KLN50}
    \end{subfigure}
    % \captionsetup{font=small}
    \caption{Results for loss function "KL Divergence" with data sizes $N=5$ and $50$ under  $\theta_e=0.02$. $95\%$ confidence intervals are reported in the shaded area.}
    \label{fig:KL}
\end{wrapfigure}
Figure~\ref{fig:eps_linear_theta002} shows that the loss of our algorithm decreases  quickly in spite of few data. In the episodic case, the loss function decreases faster with more episodes (but the same total number of iterations), due to more collected data with more episodes. The loss function of our BR-PG method decreases more quickly in early episodes, which is shown by two differences between Figure~\ref{fig:KLN5} and Figure~\ref{fig:KLN50}. First, the 95\% confidence interval, shown in the shaded band around each curve, is narrower for  $N=50$. Second, the absolute loss of $N=50$ decreases by about 20\% compared with $N=5$. Figure~\ref{fig:KL} demonstrates the better performance of our proposed BR-PG algorithm compared to the empirical approach, where we achieve smaller loss and lower variability, for the policy mimicking task.
From Table \ref{table: frozen lake small,theta0.02} and Figure \ref{fig:eps_linear_theta002}, we can see when there are more data, the posterior distribution used in BR-PG algorithm and the MLE estimator used in the empirical approach converges to the true parameter as data size increases, which reduces to solving an MDP with known transition probability, and therefore, the optimal policies and the actual costs tend to be similar. 
% \begin{figure}
% \centering
%     \begin{subfigure}[b]{0.23\textwidth}
%         \centering
%          \includegraphics[width=\textwidth]{Figure/KLN5theta02start30.png}
%          \caption{$N=5$}
%          \label{fig:KLN5}
%     \end{subfigure}
%     \hfill
%     \begin{subfigure}[b]{0.23\textwidth}
%          \centering
%          \includegraphics[width=\textwidth]{Figure/KLN50theta02start30.png}
%          \caption{$N=50$}
%          \label{fig:KLN50}
%     \end{subfigure}
%     % \captionsetup{font=small}
%     \caption{Results for loss function "KL Divergence" with data sizes $N=5$ and $50$ under  $\theta_e=0.02$. $95\%$ confidence intervals are reported in the shaded area.}
%         % \setlength{\belowcaptionskip}{-20pt}
%     \label{fig:KL}
% \end{figure}

\section{Conclusions}\label{sec:conclusion}

In this paper, we develop a Bayesian risk approach to jointly address both epistemic and intrinsic uncertainty in the infinite-horizon MDP. For a general coherent risk measure and a general convex loss function, we design a policy gradient algorithm for the proposed formulation and demonstrate the algorithm’s convergence at a rate of 
{$\mathcal{O}(T^{-1/2}+r^{-1/2})$}. Furthermore, we establish the consistency of an online episodic extension and provide bounds on the number of iterations required to maintain a constant gradient bound {$\mathcal{O}(\epsilon)$} for each episode. The numerical experiments confirm the stationary analysis of the proposed algorithm and demonstrate the robustness of the formulation under various loss functions. 

\bibliography{refs}

\begin{thebibliography}{38}
\providecommand{\natexlab}[1]{#1}
\providecommand{\url}[1]{\texttt{#1}}
\expandafter\ifx\csname urlstyle\endcsname\relax
  \providecommand{\doi}[1]{doi: #1}\else
  \providecommand{\doi}{doi: \begingroup \urlstyle{rm}\Url}\fi

\bibitem[Agarwal et~al.(2021)Agarwal, Kakade, Lee, and Mahajan]{agarwal2021theory}
Agarwal, A., Kakade, S.~M., Lee, J.~D., and Mahajan, G.
\newblock On the theory of policy gradient methods: Optimality, approximation, and distribution shift.
\newblock \emph{Journal of Machine Learning Research}, 22\penalty0 (98):\penalty0 1--76, 2021.

\bibitem[Altman(2021)]{altman2021constrained}
Altman, E.
\newblock \emph{Constrained {Markov} decision processes}.
\newblock Routledge, 2021.

\bibitem[Artzner et~al.(1999)Artzner, Delbaen, Eber, and Heath]{artzner1999coherent}
Artzner, P., Delbaen, F., Eber, J.-M., and Heath, D.
\newblock Coherent measures of risk.
\newblock \emph{Mathematical Finance}, 9\penalty0 (3):\penalty0 203--228, 1999.

\bibitem[Bai et~al.(2023)Bai, Bedi, Agarwal, Koppel, and Aggarwal]{bai2023achieving}
Bai, Q., Bedi, A.~S., Agarwal, M., Koppel, A., and Aggarwal, V.
\newblock Achieving zero constraint violation for concave utility constrained reinforcement learning via primal-dual approach.
\newblock \emph{Journal of Artificial Intelligence Research}, 78:\penalty0 975--1016, 2023.

\bibitem[Balasubramanian \& Ghadimi(2022)Balasubramanian and Ghadimi]{balasubramanian2022zeroth}
Balasubramanian, K. and Ghadimi, S.
\newblock Zeroth-order nonconvex stochastic optimization: Handling constraints, high dimensionality, and saddle points.
\newblock \emph{Foundations of Computational Mathematics}, pp.\  1--42, 2022.

\bibitem[Barakat et~al.(2023)Barakat, Fatkhullin, and He]{barakat2023reinforcement}
Barakat, A., Fatkhullin, I., and He, N.
\newblock Reinforcement learning with general utilities: Simpler variance reduction and large state-action space.
\newblock In \emph{International Conference on Machine Learning}, pp.\  1753--1800. PMLR, 2023.

\bibitem[Borkar \& Meyn(2002)Borkar and Meyn]{borkar2002risk}
Borkar, V.~S. and Meyn, S.~P.
\newblock Risk-sensitive optimal control for {Markov} decision processes with monotone cost.
\newblock \emph{Mathematics of Operations Research}, 27\penalty0 (1):\penalty0 192--209, 2002.

\bibitem[Brezis \& Br{\'e}zis(2011)Brezis and Br{\'e}zis]{brezis2011functional}
Brezis, H. and Br{\'e}zis, H.
\newblock \emph{Functional analysis, Sobolev spaces and partial differential equations}, volume~2.
\newblock Springer, 2011.

\bibitem[Chow \& Ghavamzadeh(2014)Chow and Ghavamzadeh]{chow2014algorithms}
Chow, Y. and Ghavamzadeh, M.
\newblock Algorithms for {CVaR} optimization in {MDPs}.
\newblock In Ghahramani, Z., Welling, M., Cortes, C., Lawrence, N.~D., and Weinberger, K.~Q. (eds.), \emph{Advances in Neural Information Processing Systems}, 2014.

\bibitem[Delage \& Mannor(2010)Delage and Mannor]{delage2010percentile}
Delage, E. and Mannor, S.
\newblock Percentile optimization for {M}arkov decision processes with parameter uncertainty.
\newblock \emph{Operations research}, 58\penalty0 (1):\penalty0 203--213, 2010.

\bibitem[Hong \& Liu(2009)Hong and Liu]{hong2009simulating}
Hong, L.~J. and Liu, G.
\newblock Simulating sensitivities of conditional value at risk.
\newblock \emph{Management Science}, 55\penalty0 (2):\penalty0 281--293, 2009.

\bibitem[Howard \& Matheson(1972)Howard and Matheson]{howard1972risk}
Howard, R.~A. and Matheson, J.~E.
\newblock Risk-sensitive {Markov} decision processes.
\newblock \emph{Management science}, 18\penalty0 (7):\penalty0 356--369, 1972.

\bibitem[Iyengar(2005)]{iyengar2005robust}
Iyengar, G.~N.
\newblock Robust dynamic programming.
\newblock \emph{Mathematics of Operations Research}, 30\penalty0 (2):\penalty0 257--280, 2005.

\bibitem[Lin et~al.(2022)Lin, Ren, and Zhou]{lin2022bayesian}
Lin, Y., Ren, Y., and Zhou, E.
\newblock Bayesian risk {Markov} decision processes.
\newblock In \emph{Advances in Neural Information Processing Systems}, volume~35, pp.\  17430--17442, 2022.

\bibitem[Liu et~al.(2022)Liu, Bai, Blanchet, Dong, Xu, Zhou, and Zhou]{liu2022distributionally}
Liu, Z., Bai, Q., Blanchet, J., Dong, P., Xu, W., Zhou, Z., and Zhou, Z.
\newblock Distributionally robust $ q $-learning.
\newblock In \emph{International Conference on Machine Learning}, pp.\  13623--13643. PMLR, 2022.

\bibitem[Mannor \& Tsitsiklis(2011)Mannor and Tsitsiklis]{mannor2011mean}
Mannor, S. and Tsitsiklis, J.~N.
\newblock Mean-variance optimization in {Markov} decision processes.
\newblock In \emph{Proceedings of the 28th International Conference on International Conference on Machine Learning}, pp.\  177--184, 2011.

\bibitem[Milgrom \& Segal(2002)Milgrom and Segal]{milgrom2002envelope}
Milgrom, P. and Segal, I.
\newblock Envelope theorems for arbitrary choice sets.
\newblock \emph{Econometrica}, 70\penalty0 (2):\penalty0 583--601, 2002.

\bibitem[Moon(2020)]{moon2020generalized}
Moon, J.
\newblock Generalized risk-sensitive optimal control and hamilton--jacobi--bellman equation.
\newblock \emph{IEEE Transactions on Automatic Control}, 66\penalty0 (5):\penalty0 2319--2325, 2020.

\bibitem[Nesterov \& Spokoiny(2017)Nesterov and Spokoiny]{nesterov2017random}
Nesterov, Y. and Spokoiny, V.
\newblock Random gradient-free minimization of convex functions.
\newblock \emph{Foundations of Computational Mathematics}, 17:\penalty0 527--566, 2017.

\bibitem[Nilim \& Ghaoui(2004)Nilim and Ghaoui]{nilim2003robustness}
Nilim, A. and Ghaoui, L.
\newblock Robustness in {M}arkov decision problems with uncertain transition matrices.
\newblock In Thrun, S., Saul, L., and Sch\"{o}lkopf, B. (eds.), \emph{Advances in Neural Information Processing Systems}, 2004.

\bibitem[Papini et~al.(2018)Papini, Binaghi, Canonaco, Pirotta, and Restelli]{papini2018stochastic}
Papini, M., Binaghi, D., Canonaco, G., Pirotta, M., and Restelli, M.
\newblock Stochastic variance-reduced policy gradient.
\newblock In \emph{International conference on machine learning}, pp.\  4026--4035. PMLR, 2018.

\bibitem[Pennings \& Smidts(2003)Pennings and Smidts]{pennings2003shape}
Pennings, J.~M. and Smidts, A.
\newblock The shape of utility functions and organizational behavior.
\newblock \emph{Management Science}, 49\penalty0 (9):\penalty0 1251--1263, 2003.

\bibitem[Petrik \& Russel(2019)Petrik and Russel]{petrik2019beyond}
Petrik, M. and Russel, R.~H.
\newblock Beyond confidence regions: Tight {B}ayesian ambiguity sets for robust mdps.
\newblock In Wallach, H., Larochelle, H., Beygelzimer, A., d\textquotesingle Alch\'{e}-Buc, F., Fox, E., and Garnett, R. (eds.), \emph{Advances in Neural Information Processing Systems}, volume~32, 2019.

\bibitem[Petrik \& Subramanian(2012)Petrik and Subramanian]{petrik2012an}
Petrik, M. and Subramanian, D.
\newblock An approximate solution method for large risk-averse {Markov} decision processes.
\newblock In \emph{Proceedings of the Twenty-Eighth Conference on Uncertainty in Artificial Intelligence}, pp.\  805–814, 2012.

\bibitem[Pichler et~al.(2022)Pichler, Liu, and Shapiro]{pichler2022risk}
Pichler, A., Liu, R.~P., and Shapiro, A.
\newblock Risk-averse stochastic programming: Time consistency and optimal stopping.
\newblock \emph{Operations Research}, 70\penalty0 (4):\penalty0 2439--2455, 2022.

\bibitem[Ravichandiran(2018)]{ravichandiran2018hands}
Ravichandiran, S.
\newblock \emph{Hands-on reinforcement learning with Python: master reinforcement and deep reinforcement learning using OpenAI gym and tensorFlow}.
\newblock Packt Publishing Ltd, 2018.

\bibitem[Ruszczy{\'n}ski(2010)]{ruszczynski2010risk}
Ruszczy{\'n}ski, A.
\newblock Risk-averse dynamic programming for {Markov} decision processes.
\newblock \emph{Mathematical programming}, 125:\penalty0 235--261, 2010.

\bibitem[Shapiro(2012)]{shapiro2012minimax}
Shapiro, A.
\newblock Minimax and risk averse multistage stochastic programming.
\newblock \emph{European Journal of Operational Research}, 219\penalty0 (3):\penalty0 719--726, 2012.

\bibitem[Shapiro et~al.(2021)Shapiro, Dentcheva, and Ruszczynski]{shapiro2021lectures}
Shapiro, A., Dentcheva, D., and Ruszczynski, A.
\newblock \emph{Lectures on stochastic programming: modeling and theory}.
\newblock SIAM, 2021.

\bibitem[Shapiro et~al.(2023)Shapiro, Zhou, and Lin]{shapiro2023bayesian}
Shapiro, A., Zhou, E., and Lin, Y.
\newblock Bayesian distributionally robust optimization.
\newblock \emph{SIAM Journal on Optimization}, 33\penalty0 (2):\penalty0 1279--1304, 2023.

\bibitem[Sutton et~al.(1999)Sutton, McAllester, Singh, and Mansour]{sutton1999policy}
Sutton, R.~S., McAllester, D., Singh, S., and Mansour, Y.
\newblock Policy gradient methods for reinforcement learning with function approximation.
\newblock In Solla, S.~A., Leen, T.~K., and Müller, K.-R. (eds.), \emph{Advances in Neural Information Processing Systems}, pp.\  1057--1063, 1999.

\bibitem[Tamar et~al.(2015)Tamar, Chow, Ghavamzadeh, and Mannor]{tamar2015policy}
Tamar, A., Chow, Y., Ghavamzadeh, M., and Mannor, S.
\newblock Policy gradient for coherent risk measures.
\newblock \emph{Advances in neural information processing systems}, 28, 2015.

\bibitem[Wiesemann et~al.(2013)Wiesemann, Kuhn, and Rustem]{wiesemann2013robust}
Wiesemann, W., Kuhn, D., and Rustem, B.
\newblock Robust {M}arkov decision processes.
\newblock \emph{Mathematics of Operations Research}, 38\penalty0 (1):\penalty0 153--183, 2013.

\bibitem[Ying et~al.(2023)Ying, Guo, Ding, Lavaei, and Shen]{ying2023policy}
Ying, D., Guo, M.~A., Ding, Y., Lavaei, J., and Shen, Z.-J.
\newblock Policy-based primal-dual methods for convex constrained markov decision processes.
\newblock In \emph{Proceedings of the AAAI Conference on Artificial Intelligence}, volume~37, pp.\  10963--10971, 2023.

\bibitem[Ying et~al.(2024)Ying, Zhang, Ding, Koppel, and Lavaei]{ying2024scalable}
Ying, D., Zhang, Y., Ding, Y., Koppel, A., and Lavaei, J.
\newblock Scalable primal-dual actor-critic method for safe multi-agent rl with general utilities.
\newblock \emph{Advances in Neural Information Processing Systems}, 36, 2024.

\bibitem[Zhang et~al.(2020)Zhang, Koppel, Bedi, Szepesvari, and Wang]{zhang2020variational}
Zhang, J., Koppel, A., Bedi, A.~S., Szepesvari, C., and Wang, M.
\newblock Variational policy gradient method for reinforcement learning with general utilities.
\newblock \emph{Advances in Neural Information Processing Systems}, pp.\  4572--4583, 2020.

\bibitem[Zhang et~al.(2021)Zhang, Ni, Szepesvari, Wang, et~al.]{zhang2021convergence}
Zhang, J., Ni, C., Szepesvari, C., Wang, M., et~al.
\newblock On the convergence and sample efficiency of variance-reduced policy gradient method.
\newblock \emph{Advances in Neural Information Processing Systems}, 34:\penalty0 2228--2240, 2021.

\bibitem[Zhang et~al.(2022)Zhang, Bedi, Wang, and Koppel]{zhang2022multi}
Zhang, J., Bedi, A.~S., Wang, M., and Koppel, A.
\newblock Multi-agent reinforcement learning with general utilities via decentralized shadow reward actor-critic.
\newblock In \emph{Proceedings of the AAAI Conference on Artificial Intelligence}, volume~36, pp.\  9031--9039, 2022.

\end{thebibliography}
\bibliographystyle{icml2025}

%%%%%%%%%%%%%%%%%%%%%%%%%%%%%%%%%%%%%%%%%%%%%%%%%%%%%%%%%%%%%%%%%%%%%%%%%%%%%%%
%%%%%%%%%%%%%%%%%%%%%%%%%%%%%%%%%%%%%%%%%%%%%%%%%%%%%%%%%%%%%%%%%%%%%%%%%%%%%%%
% APPENDIX
%%%%%%%%%%%%%%%%%%%%%%%%%%%%%%%%%%%%%%%%%%%%%%%%%%%%%%%%%%%%%%%%%%%%%%%%%%%%%%%
%%%%%%%%%%%%%%%%%%%%%%%%%%%%%%%%%%%%%%%%%%%%%%%%%%%%%%%%%%%%%%%%%%%%%%%%%%%%%%%
\newpage
\appendix
\onecolumn

%%%%%%%%%%%%%%%%%%%%%%%%%%%%%%%%%%%%%%%%%%%%%%%%%%%%%%%%%%%%%%%%%%%%%%%%%%%%%%%
%%%%%%%%%%%%%%%%%%%%%%%%%%%%%%%%%%%%%%%%%%%%%%%%%%%%%%%%%%%%%%%%%%%%%%%%%%%%%%%
\section{Review on convex RL}
\label{appendix:convexRL}
% \color{blue}
Our problem is highly relevant to convex RL, which generalizes cumulative reward on a convex general-utility objective instead of cumulative reward. 
Specifically, our problem is closely tied to convex RL, which extends the traditional cumulative reward framework to a convex general-utility objective. Prior research has explored policy gradient methods to address convex RL. For instance, \citep{zhang2020variational}  demonstrates that the policy gradient of convex RL can be formulated as a min-max optimization problem. To reduce estimator variance, \citep{zhang2021convergence} introduces an off-policy policy gradient estimator that leverages mini-batch techniques and truncation mechanisms, while \citep{barakat2023reinforcement} employs a recursive approach to handle large state-action spaces. In the domain of multi-agent convex RL, \citep{zhang2022multi} assumes global state observability and proposes a trajectory-based actor-critic method. Recent studies have also focused on safe convex RL, where the objective is to maximize a convex utility function under convex safety constraints. For example, \citep{ying2023policy} develops a primal-dual algorithm with strong guarantees on the optimality gap and constraint violations, achieving an 
$\mathcal{O}(1/\epsilon^3)$
 bound in the convex-concave case with zero constraint violation. Building on this, \citep{bai2023achieving} improves the bound to 
$\mathcal{O}(1/\epsilon^2)$. Furthermore, \citep{ying2024scalable} extends the primal-dual framework to multi-agent convex safe RL.

\section{Discussion on Non-convex Structure}
\label{appendix:non-convex}

For any fixed environment parameter $\theta$, the set of occupancy measure is convex, which makes global convergence achievable. But the global convergence cannot be achieved under the Bayesian setting because the set of occupancy measure is nonconvex!  Consider the simple case that $\Theta=\{\theta_1,\theta_2\}$ and the state space and action space are finite. Then for any fixed policy parameter $\alpha$ and environment parameter $\theta_i$, the occupancy measure $\lambda(\alpha,\theta_i)$ is a $|S||A|$-dimensional vector. Under any fixed $\theta_i$, for any $t\in[0,1]$ and policy parameters $\alpha_1,\alpha_2$, the convex combination $t\lambda(\alpha_1,\theta_i)+(1-t)\lambda(\alpha_2,\theta_i)$ is also an occupancy measure corresponding to some policy $\alpha_3$ under $\theta_i$. This hidden convexity can be utilized to achieve global convergence for a fixed $\theta$. However, the occupancy measure is a $2|S||A|$-dimensional vector under the Bayesian setting, and we need an additional constraint to guarantee that the occupancy measures under different environments correspond to the same policy $\alpha$:
\begin{equation*}
    \frac{\lambda(\alpha,\theta_1)_{s,a}}{\sum_{a^{\prime}} \lambda(\alpha,\theta_1)_{s,a'}}= \frac{\lambda(\alpha,\theta_2)_{s,a}}{\sum_{a^{\prime}} \lambda(\alpha,\theta_2)_{s,a'}},\forall s,a.
\end{equation*}

This constraint means that the agent chooses the action $a$ at state $s$ with the same probability under two environments, which makes the  set of occupancy measure nonconvex. In other words, for any $t\in[0,1],\alpha_1,\alpha_2$, there may not exist a policy $\alpha_3$ 
 such that $t(\lambda(\alpha_1,\theta_1),\lambda(\alpha_1,\theta_2))+(1-t)(\lambda(\alpha_2,\theta_1),\lambda(\alpha_2,\theta_2))=(\lambda(\alpha_3,\theta_1),\lambda(\alpha_3,\theta_2)).$
 Thus, the global convergence cannot be achieved due to the lack of intrinsic convexity under the Bayesian setting with a fixed posterior. 
 However, as the data size $N$ increases, the posterior distribution converges to a Dirac measure, which is a point mass at the true parameter $\theta^*$. Consequently, the performance of the optimal policy for the posterior $\mu_N$ converges to the optimal policy under the true environment, as demonstrated in Theorem \ref{thm:consistency}.  What's more, the global optimality gap will converge to any accuracy $\epsilon$ when the data size $N$ increases, as shown in Theorem \ref{thm:two_episode}.  In other words, global convergence can be achieved in our episodic setting (when the data size $N$ increases to infinity).
\color{black}
\newpage

\section{Discussion on Assumption \ref{assum:Lipschitz_F_C}}
\label{appendix:Lipshitz}

Recall that $C(\alpha, \theta):=F\left(\lambda^{\pi_\alpha, P_\theta}, P_\theta\right)$. By chain rule, $\nabla_\alpha C(\alpha, \theta)=\nabla_\lambda F \cdot \nabla_\alpha \lambda^{\pi_\alpha}$. Similar to classical Policy Gradient Theorem, it holds that $\nabla_\alpha \lambda^{\pi_\alpha}(s, a)=\lambda^{\pi_\alpha}(s, a) \nabla_\alpha \log \pi_\alpha(a \mid s)$. Thus, the behavior of $\nabla_\alpha C(\alpha, \theta)$ depends on the regularity of both $\nabla_\lambda F$ and $\nabla_\alpha \log \pi_\alpha(a \mid s)$. In classical RL analysis like \citet{agarwal2021theory,papini2018stochastic}, one typically assumes that $\nabla_\alpha \log \pi_\alpha(a \mid s)$ is either Lipschitz continuous or uniformly bounded, and per-step rewards $r(s, a)$ are assumed to be bounded for all ( $s, a$ ) pairs, which together implies the smoothness of the expected return with respect to the policy parameter. In our setting, we replace the usual expected return with a general convex loss function $F$. Consequently, we must impose smoothness conditions on $F$. Specifically, if $F(\lambda)$ is Lipschitz continuous with Lipschitz-continuous gradient, and $\nabla_\alpha \log \pi_\alpha(a \mid s)$ is Lipschitz continuous and bounded, then Assumption 3.2 is satisfied.

 We can demonstrate that the classic Linear-Quadratic Regulator (LQR), perhaps the simplest continuous-action benchmark, satisfies all of these assumptions.

The dynamics is
\begin{align*}
s_{t+1}=A s_t+B a_t+w_t, \quad w_t \sim \mathcal{N}\left(0, \Sigma_w\right)
\end{align*}
and the policy is Gausssian
\begin{align*}
a_t \sim \pi_\alpha\left(\cdot \mid s_t\right)=\mathcal{N}\left(K s_t, \Sigma_a\right)
\end{align*}
with parameter $\alpha=\operatorname{vec}(K)$. Define the loss function $F$ to be any convex function in $\lambda$. Since both the transition and the policy are affine and Gaussian, the joint law ($s_t, a_t$) is Gaussian at every $t$. Then the occupancy measure is a discounted summation of Gaussian distributions.
Recall that  the behavior of $\nabla_\alpha C(\alpha, \theta)$ depends on the regularity of both $\nabla_\lambda F$ and $\nabla_\alpha \log \pi_\alpha(a \mid s)$.  In the example of LQR, the first and second derivatives of $\log \pi_\alpha(a \mid s)$ are bounded since the policy is Gaussian.
Specifically, if $F(\lambda)$ is Lipschitz continuous with Lipschitz-continuous gradient, and since $\nabla_\alpha \log \pi_\alpha(a \mid s)$ is Lipschitz continuous and bounded in this LQR problem, then Assumptions 3.2 (2)(3) are satisfied.

In the example of $\mathrm{LQR}, \Theta$ is a set containing possible $A, B, \Sigma_w$. Assumptions 3.2 (4) is about the compactness and convexity of $\Theta$, which is not strong. One sufficient condition for Assumption 4.1 to hold is that $\xi^*$ is bounded on $\Theta$. As $\Theta$ is a compact and convex set, it is not a strong condition.

\section{Discussion on Differences between Our Method and \citet{tamar2015policy}}

\label{appendix:difference_tamar}
 We want to obtain $\nabla_\alpha\left[\rho_{\theta \sim \mu_N}(C(\alpha, \theta))\right]$, where the derivative is taken with respect to the policy parameter $\alpha$ in the general loss function $C(\alpha, \theta)$. On the other hand, the problem in \citet{tamar2015policy} is how to get $\nabla_\alpha\left[\rho_{\theta \sim \mu_N(\alpha)}(D(\theta))\right]$ for a random variable $D(\theta)$, where the derivative is taken with respect to $\alpha$ in the distribution $\mu_N$. The difference in settings essentially leads to different forms of Lagrangians and causes the failure to apply their results to our setting. What's more, $\Theta$ is a finite set in their setting, but $\Theta$ can be an uncountable continuous subset of some $\mathbb{R}^d$ in our setting. As a result, when we use the Envelope Theorem to prove the result, we are facing an infinite-dimensional optimization problem over functions, which is a much harder problem than their finite-dimensional optimization problem over vectors. Briefly speaking, in our proof we ensure differentiability and integrability conditions hold uniformly, and construct a separable set of disturbation functions as the domain for Lagrangians.

The original envelope theorem, as presented by \citet{milgrom2002envelope}, primarily addresses finite-dimensional parameter spaces. \citet{tamar2015policy} utilize a discrete-parameter space framework for policy gradients in risk-sensitive Markov Decision Processes (MDPs), restricting their applicability to problems with finite and discrete parameter settings.
Our extension generalizes the envelope theorem to handle continuous parameter spaces, significantly broadening the applicability of policy gradient methods to a wider class of problems, such as those involving continuous uncertainty sets or continuous Bayesian posterior distributions over model parameters. Specifically, we address the additional complexities introduced by functional optimization in infinite-dimensional spaces, which involves ensuring differentiability and integrability conditions hold uniformly.

This generalization is nontrivial as it requires overcoming challenges associated with infinite-dimensional optimization, such as ensuring the boundedness and continuity of gradients and validating strong duality under more complex integrative constraints. As such, our contribution facilitates the development of theoretically sound policy gradient methods capable of addressing a broader and more practical range of MDP formulations where uncertainty is represented continuously. This makes our methodology independently interesting to researchers in stochastic control, reinforcement learning, and risk-sensitive optimization.

\section{Algorithms}
\label{appendix:algorithms}
\begin{algorithm}[ht]
\caption{Bayesian Risk Policy Gradient (BR-PG)}
\label{alg:sgd_coherent}
\begin{algorithmic}
\State \textbf{input}: Initial $\alpha_0$, data $\zeta^{(N)}$ of size $N$, prior distribution $\mu_0(\theta)$, iteration number $T$;
% truncation horizon $K$;
% \State  $\alpha_T$;
\State Calculate the posterior $\mu_N(\theta)=\frac{P_{\theta}(\zeta^{(N)})\mu_0(\theta)}{\int_{\Theta} P_{\theta'}(\zeta^{(N)})\mu_0(\theta')d\theta'}$;
\For{ $t =0$ to $T-1$}
\State Sample $\{\theta_k^t\}_{k=1}^{r}$ from $\mu_N(\theta)$;
\State Use the closed-form expression or solve  \eqref{eq:max_min_SAA} to get $\xi^*(\theta_k^t)$;

\State Solve \eqref{eq:min_max_nablaC} to get $\nabla_\alpha C\left(\alpha_t,\theta_k^t\right)
$ for  $k=1.\dots,r$;
% \For{$k =1$ to $r$}
% \State Solve the problem \eqref{eq:min_max_nablaC} to get $\nabla_\alpha C\left(\alpha_t,\theta_k^t\right)
% $
% \EndFor
\State $\widehat{g_t}:=\frac{1}{r}\sum_{k=1}^{r}\xi^*(\theta_k^t)\nabla_\alpha C\left(\alpha,\theta_k^t\right)$;
\State $\alpha_{t+1}=\operatorname{Proj}_{W}\left(\alpha_t-\frac{1}{\eta_t}\widehat{g}_t\right)$;
\EndFor
\State \textbf{output}: Choose  $\alpha_{\text{out}}$ uniformly from $\alpha_0,\dots,\alpha_{T-1}$.
\end{algorithmic}
\end{algorithm}

\begin{algorithm}[ht]
\caption{ Episodic BR-PG}
\label{alg:sgd_episodic}
\begin{algorithmic}
\State \textbf{input}: Initial $\alpha_0$, prior distribution $\mu_0(\theta)$,  total episode number $N$.
% truncation horizon $K$;
% \State  $\alpha_T$;
\State Deploy policy $\pi(\alpha_0)$ to gain the initial data set $\zeta^{(1)}$.
\For{ $i =1$ to $N$}
\If{$i=1$}
\State $\alpha_{i,0}=\alpha_0$
\Else
\State Let $\alpha_{i,0}$ to be uniformly chosen from $\alpha_{i-1,0},\dots,\alpha_{i-1,t_{i-1}-1}$;
\EndIf
% \If{$i<N$}
\State Deploy policy $\pi(\alpha_{i,0})$ to gain a  data set $\zeta^{(i)}$.%, which incorporates the old data set $\zeta^{(i)}$ .
% \EndIf
\State Calculate the posterior $\mu_i(\theta)=\frac{P_{\theta}(\zeta^{(i)})\mu_{i-1}(\theta)}{\int_{\theta'} P_{\theta'}(\zeta^{(i)})\mu_{i-1}(\theta')}$;
\State Use Algorithm \ref{alg:sgd_coherent} with $t_i$ iterations and initial point $\alpha_{i,0}$ to generate the policy parameter sequence $\alpha_{i,1},\cdots,\alpha_{i,t_i}$.
\EndFor
\State \textbf{output}: Let $\alpha_{\text{out}}$ to be randomly and uniformly chosen from $ \alpha_{N,0},\dots,\alpha_{N,t_N-1}$.
\end{algorithmic}
\end{algorithm}

\newpage
\section{Proof Details}

\subsection{Definition of Radon Measure}
\label{appendix: radon}
\begin{definition}
    $\mu_N$ is a Radon measure on $\Theta$ if (i) $\mu_N(\Theta)$ is finite, (ii) for all Borel set $E\subseteq \Theta$, we have $\mu_N(E)=\inf\{\mu_N(U):E\subseteq U, U \text{ is open}\}$ and $\mu_N(E)=\sup\{\mu_N(K):K\subseteq E,K \text{ is compact}\}.$ 
\end{definition}
For a continuous parameter space $\Theta$, if the prior is a continuous distribution and the likelihood function is continuous in $\theta$, then the posterior is Radon. And it always holds for discrete cases.  Thus it hold in most cases that we may care about, and most common probability distributions are Radon Measures.

\subsection{Proof of Theorem \ref{thm:gradient_coherent}}
\label{appendix:gradient_coherent}

\begin{proof}
\begin{equation*}
    \begin{split}
  \mathcal{U}\left(\mu_N\right)=\{\xi :& g_e\left(\xi,\mu_N\right)=0, \forall e \in \mathcal{E}, f_i\left(\xi, \mu_N\right) \leq 0, \forall i \in \mathcal{I},\\
  &\int_{\theta \in \Theta} \xi(\theta) \mu_N(\theta)=1, \xi(\theta) \geq 0,\|\xi\|_q\le B_q\}. 
\end{split}
\end{equation*}

Define the Lagrangian:
\begin{equation}
L_\alpha(\xi,\lambda^{\mathcal{I}},\lambda^{\mathcal{E}})=\int_{\theta \in \Theta} \xi(\theta) \mu_N(\theta) C(\alpha,\theta)-\sum_{i \in \mathcal{I}} \lambda^{\mathcal{I}}(i) f_i\left(\xi, \mu_N\right)-\sum_{e \in \mathcal{E}} \lambda^{\mathcal{E}}(e) g_e\left(\xi, \mu_N\right),
\end{equation}
and a subtly relaxed envelope
\begin{equation*}
\begin{split}
  \mathcal{U}'\left(\mu_N\right)=\{\xi : 
  &\int_{\theta \in \Theta} \xi(\theta) \mu_N(\theta)=1, \xi(\theta) \geq 0,,
\|\xi\|_q \le B_q\}.  
\end{split}
\end{equation*}

As mentioned before, we can rewrite the objective as the value of a max-min problem in \eqref{eq:dual_problem}
\[\rho_{\theta\sim\mu_N}(C(\alpha,\theta))=\max_{\xi\in \mathcal{U}'(\mu_N)}\min _{\lambda^{\mathcal{I}} \geq 0,\lambda^{\mathcal{E}} } L_\alpha(\xi,\lambda^{\mathcal{I}},\lambda^{\mathcal{E}}).
\]

Two things deserved to be noticed:
(i) Slater's condition holds in the primal optimization problem \eqref{eq:dual_problem} by Definition \ref{assum: envelop}.
(ii) $L_\theta\left(\xi, \lambda^{\mathcal{I}},\lambda^{\mathcal{E}}\right)$ is concave in $\xi$ and convex in $(\lambda^{\mathcal{I}},\lambda^{\mathcal{E}}).$
Then strong duality holds for \eqref{eq:dual_problem}.
\begin{equation}
\label{eq:saddle_problem}
    \begin{split}
        \rho_{\theta\sim\mu_N}(C(\alpha,\theta))&=\max_{\xi\in \mathcal{U}'(\mu_N)}\min _{\lambda^{\mathcal{I}} \geq 0,\lambda^{\mathcal{E}} } L_\alpha(\xi,\lambda^{\mathcal{I}},\lambda^{\mathcal{E}})\\
&=\min _{\lambda^{\mathcal{I}} \geq 0,\lambda^{\mathcal{E}} }\max_{\xi\in \mathcal{U}'(\mu_N)}L_\alpha(\xi,\lambda^{\mathcal{I}},\lambda^{\mathcal{E}})
     \end{split}
\end{equation}

As  $\nabla_\alpha C(\alpha,\theta)$ is uniformly bounded for all $\theta$ and $\alpha$, we have $\nabla_\alpha L_\alpha(\xi,\lambda^{\mathcal{I}},\lambda^{\mathcal{E}})$ is uniformly bounded w.r.t $\alpha$  and continuous at all$(\xi,\lambda^{\mathcal{I}},\lambda^{\mathcal{E}})$. Then we have
$L_\alpha(\xi,\lambda^{\mathcal{I}},\lambda^{\mathcal{E}})$ is absolutely continuous w.r.t $\alpha$ for all $(\xi,\lambda^{\mathcal{I}},\lambda^{\mathcal{E}})$.
Since $\nabla_\alpha^2 C(\alpha,\theta)$ is uniformly bounded for all $\theta$ and $\alpha$, we have $\{L_\alpha(\xi,\lambda^{\mathcal{I}},\lambda^{\mathcal{E}})\}_{(\xi, \lambda^{\mathcal{I}}, \lambda^{\mathcal{E}})}$ is equi-differentiable in $\alpha$.

As $\Theta$ is compact and convex,   $\Theta$ is a separable metric space with Euclidean metric  and its Borel sigma algebra. Then $(\Theta,\mu_N)$ is a  separable metric measure space.  By  Theorem 4.13 \citep{brezis2011functional}, $L^q(\Theta,\mu_N)$ is separable. Then $ \mathcal{U}'\left(\mu_N\right)=\{\xi \in L^q(\Theta,\mu_N): \int_{\theta \in \Theta} \xi(\theta) \mu_N(\theta)=1, \xi(\theta) \geq 0,,
\|\xi\|_q \le B_q\}  $ is separable.

Define the set of saddle point for \eqref{eq:saddle_problem} by $X^*=\arg \max_{\xi \in \mathcal{U}'\left(\mu_N\right)}\min _{\lambda^{\mathcal{I}} \geq 0,\lambda^{\mathcal{E}} } L_\alpha(\xi,\lambda^{\mathcal{I}},\lambda^{\mathcal{E}})$
and $Y^*=\arg \min _{\lambda^{\mathcal{I}} \geq 0,\lambda^{\mathcal{E}} } \max_{\xi \in \mathcal{U}'\left(\mu_N\right)}L_\alpha(\xi,\lambda^{\mathcal{I}},\lambda^{\mathcal{E}})$.

 Then for every selection of saddle point $\left(\xi_\alpha^*, \lambda_\alpha^{*, \mathcal{E}}, \lambda_\alpha^{*, \mathcal{I}}\right) \in X^*\times Y^*$, the Envelope theorem for saddle-point problems ( Theorem 4\citep{milgrom2002envelope} ) shows that

 \begin{equation}
     \begin{split}
     \nabla_\alpha \rho_{\theta\sim\mu_N}(C(\alpha,\theta))=
        &\nabla_\alpha \max_{\xi \in \mathcal{U}'\left(\mu_N\right)}\min _{\lambda^{\mathcal{I}} \geq 0,\lambda^{\mathcal{E}} } L_\alpha(\xi,\lambda^{\mathcal{I}},\lambda^{\mathcal{E}})\\
        &=\left.\nabla_\alpha L_\alpha(\xi,\lambda^{\mathcal{I}},\lambda^{\mathcal{E}})\right|_{\left(\xi_\alpha^*, \lambda_\alpha^{*, \mathcal{E}}, \lambda_\alpha^{*, \mathcal{I}}\right)} \\
        &=\int_{\theta \in \Theta} \xi_\alpha^*(\theta) \mu_N(\theta)\nabla_\alpha C(\alpha,\theta)
     \end{split}
 \end{equation}
 \end{proof}

\subsection{Proof of Lemma~\ref{thm:nablaC}}
\label{appendix:nablaC_solve}
\begin{proof}
Here is a brief proof sketch, and the full proof can be found in the proof of Theorem 3.1\citep{zhang2020variational}.  For a convex function, the  conjugate of the conjugate is itself. Notice that $V(\alpha ; z)+\delta \nabla_\alpha V(\alpha ; z)^{\top} x-F^*(z) =\langle z, \lambda(\alpha,\theta)+\delta \nabla_\alpha \lambda(\alpha,\theta) x\rangle-F^*(z)  $.
 Then we have $\sup_{z\in \mathbb{R}^{SA}}V(\alpha ; z)+\delta \nabla_\alpha V(\alpha ; z)^{\top} x-F^*(z)=F(\lambda(\alpha,\theta)+\delta \nabla_\alpha \lambda(\alpha,\theta) x) $. By the first-order condition, we have \[\underset{x\in \mathbb{R}^{SA}}{\operatorname{argmin}} F(\lambda(\alpha,\theta)+\delta \nabla_\alpha \lambda(\alpha,\theta) x)+\frac{\delta}{2}\|x\|_2^2=- \nabla F\left(\lambda(\alpha,\theta)+\delta \nabla_\alpha \lambda(\alpha,\theta) x\right)\nabla_\alpha \lambda(\alpha,\theta) x).\]

By letting $\delta \to 0^+$ and using the chain rule, we get the result \eqref{eq:min_max_nablaC}.
\end{proof}

\subsubsection{Algorithm for solving Theorem \ref{thm:nablaC}}
\label{subsec:solve_min_max}
\textbf{Estimate $V(\alpha,z)$:} Recall that we consider an offline setting where the transition kernel $P_\theta$ is assumed to be known for any given $\theta$. For any fixed transition kernel $P_\theta$ and policy $\pi_\alpha$, we can estimate the occupancy measure by making a truncation $K$ in the definition of occupancy measure in \eqref{eq:def_occu}:
\begin{equation*}
 \widehat{\lambda_{s a}^{\pi,P}} =\sum_{t=0}^{K} \gamma^t \cdot \mathbb{P}\left(s_t=s, a_t=a \mid \pi, s_0 \sim \tau,P \right)
\end{equation*}
with the error $\|\widehat \lambda -\lambda\|_1\le \epsilon_\lambda:=\gamma^K/(1-\gamma) $. This error can be made arbitrarily small by increasing $K$, thus we assume that we can exactly compute occupancy measure. After computing the occupancy measure, $V(\alpha; z) = \langle z, \lambda\rangle$.

\textbf{Estimate $\nabla_\alpha V(\alpha,z)$:}
The policy gradient theorem \citep{sutton1999policy} shows that \[\nabla_\alpha V(\alpha ; z)=\mathbb{E}^{\pi_\alpha}\left[\sum_{t=0}^{\infty} \gamma^t Q^{\pi_\alpha}\left(s_t, a_t ; z\right) \cdot \nabla_\alpha \log \pi_\alpha\left(a_t \mid s_t\right)\right]\]
where $Q^\pi(s, a ; z):=\mathbb{E}^\pi\left[\sum_{t=0}^\infty \gamma^t z\left(s_t, a_t\right) \mid s_0=s, a_0=a, a_t \sim \pi\left(\cdot \mid s_t\right)\right]$ satisfying the Bellman equation
\begin{equation}
\label{eq:bellman_Q}
       Q^\pi(s, a ; z)=\mathbb{E}[z(s,a)]+\sum_{s'\in\mathcal{S}}\sum_{a'\in\mathcal{A}} P_\theta(s'|s,a)\pi(a'|s')Q^\pi(s', a' ; z).
\end{equation}

For any given $\theta$, policy $\pi$ and cost function $z$, we can solve the Bellman equation \eqref{eq:bellman_Q} exactly to get $Q(\cdot,\cdot)$. It can be seen that $\nabla_\alpha V(\alpha ; z)$ is a linear function of $\lambda$:
\[
\nabla_\alpha V(\alpha ; z)=\sum_{s\in\mathcal{S}}\sum_{a\in\mathcal{A}} Q(s,a)\dot\nabla_\alpha \log \pi_\alpha\left(a \mid s\right)\lambda(s,a).
\]

% \textbf{Solve the min-max problem \eqref{eq:min_max_nablaC}:} Note that the objective in the min-max problem \eqref{eq:min_max_nablaC} is concave in $z$ and strongly convex in $x$, and the saddle point problem of which has been widely researched. Any algorithm that solve \eqref{eq:min_max_nablaC} will satisfy our demand. For example, an alternative gradient method is provided in Appendix \ref{subsec:solve_min_max}.

For any $\theta$, policy $\pi$ and cost function $z$, we can calculate the $Q$ value function by solving the Bellman equation:
\begin{equation*}
       Q^\pi(s, a ; z)=\mathbb{E}[z(s,a)]+\sum_{s'\in\mathcal{S}}\sum_{a'\in\mathcal{A}} P_\theta(s'|s,a)\pi(a'|s')Q^\pi(s', a' ; z) 
\end{equation*}

Then we can use Algorithm \ref{alg:alternative_gd} to solve \eqref{eq:min_max_nablaC} in Lemma~\ref{thm:nablaC}. It should be noticed that $\delta \nabla_\alpha V(\alpha ; z)^{\top} x=\mathcal{O}(\delta)$ is omitted when calculating the gradient for $z$ as $\delta\to 0$. Thus we omit this term when calculating the gradient for $z$. To evaluate the integral $\int_{\theta \in \Theta} \xi_\alpha^*(\theta) \mu_N(\theta)\nabla_\alpha C(\alpha,\theta)$, we sample i.i.d $\theta_k$ from $\mu_N$ for $k=1,\dots,r$, then we can construct the policy gradient estimator
\begin{equation*}
    \nabla_\alpha \rho_{\theta\sim\mu_N}(C(\alpha,\theta))\approx \widehat{g}(\alpha):=\frac{1}{r}\sum_{k=1}^{r}\xi^*(\theta_k)\nabla_\alpha C\left(\alpha,\theta_k\right).
\end{equation*}
\begin{algorithm}
\caption{Alternative Gradient Descent Method}
\label{alg:alternative_gd}
\begin{algorithmic}

\State \textbf{input}: initial $z_0,x_0$, step sizes ${a_t},{b_t}$, iteration number $T$, transition kernel parameter $\theta$, policy parameter $\alpha$;

\For{$t =0$ to $T-1$}
\State$
z_{t+1}=z_t+a_t [\lambda(\alpha,\theta)-\nabla F^*(z_t)]
$
\State $x_{t+1}=x_t-b_t[\nabla_\alpha V(\alpha ; z)+x_t]$, where
$\nabla_\alpha V(\alpha ; z)=\sum_{s,a} Q(s,a)\dot\nabla_\alpha \log \pi_\alpha\left(a \mid s\right)\lambda(s,a)$
\EndFor
\State \textbf{output}: $-x_T$.
\end{algorithmic}
\end{algorithm}

\subsection{Proof of Theorem~\ref{thm:estimator_error_coherent}}
\label{appendix:estimator_error_coherent}
% The vector form of Cramér's large deviation theorem is used to prove Theorem~\ref{thm:estimator_error_coherent}.
% \begin{lemma}(Lemma 2.1~\citep{hu2020sample})
%     \label{lem:large_deviation}
% Let $X_1, \cdots, X_n$ be i.i.d samples of zero mean random variable $X$ with finite variance $\sigma^2$. For any $\epsilon>0$, it holds
% $$
% \mathbb{P}\left(\frac{1}{n} \sum_{i=1}^n X_i \geq \epsilon\right) \leq \exp (-n I(\epsilon))
% $$
% where $I(\epsilon):=\sup _{t \in \mathbb{R}}\{t \epsilon-\log M(t)\}$ is the rate function of random variable $X$, and $M(t):=\mathbb{E} e^{t X}$ is the moment generating function of $X$. For any $\delta>0$, there exists $\epsilon_1>0$, for any $\epsilon \in\left(0, \epsilon_1\right), I(\epsilon) \geq \frac{\epsilon^2}{(2+\delta) \sigma^2}$. If $X$ is a zero-mean sub-Gaussian, then $\mathbb{P}\left(\frac{1}{n} \sum_{i=1}^n X_i \geq \epsilon\right) \leq \exp \left(-\frac{n \epsilon^2}{2 \sigma^2}\right), \forall \epsilon>0$.

% If $X$ is a zero-mean random vector in $\mathbb{R}^k$ such that $\mathbb{E}\|X\|_2^2=\sigma^2<\infty$, then for any $\delta>0$, there exists $\epsilon_1>0$, for any $\epsilon \in\left(0, \epsilon_1\right)$,
% $$
% \mathbb{P}\left(\left\|\frac{1}{n} \sum_{i=1}^n X_i\right\|_2 \geq \epsilon\right) \leq 2 k \exp \left(-\frac{n \epsilon^2}{(2+\delta) \sigma^2}\right)
% $$
% \end{lemma}

% Then we turn to prove Theorem~\ref{thm:estimator_error_coherent}.
\begin{proof}
     By Theorem \ref{thm:gradient_coherent}, the true gradient is 
  \begin{equation*}
      g(\alpha)=\int_{\theta \in \Theta} \xi_\alpha^*(\theta) \mu_N(\theta)\nabla_\alpha C(\alpha,\theta). 
  \end{equation*}
And our gradient estimator is 
\begin{equation*}
   \widehat{g}(\alpha):=\frac{1}{r}\sum_{k=1}^{r}\xi^*(\theta_k)\nabla_\alpha C\left(\alpha,\theta_k\right).
\end{equation*}
Then we have
\begin{equation*}
    \begin{split}
       \mathbb{E}\|\widehat{g}-g\|_2^2&\le \frac{1}{r} \mathbb{E}
     \|\xi^*(\theta_1)\nabla_\alpha C\left(\alpha,\theta_1\right)-\int_\Theta \xi^*(\theta)\mu_N(\theta)\nabla_\alpha C(\alpha,\theta)d\theta\|_2^2\le \frac{\sigma_\xi}{r}.
    \end{split}
\end{equation*}
\end{proof}

\subsection{Proof of Theorem \ref{thm:converge_alg}}
% \color{blue}
% Update it to stationary version
\label{appendix:convergenc_alg}
First, we make an assumption about $G$.
\begin{assumption}
\label{assum:bounded_gradient}
There exists some $L_G>0$ s.t. $g(\alpha)$ is $L_G$-Lipschitz continuous in $\alpha$.  
\end{assumption}

\begin{proof}
   For ease of notation, denote $g(\alpha_t)$ as $g_t$ and $\widehat{g}(\alpha_t)$ as $\widehat{g}_t$. By Assumption \ref{assum:bounded_gradient}, we have
\begin{equation}
\label{eq:L_G_inequality}
\begin{split}
G(\alpha) &\le  G\left(\alpha_t\right)+\left\langle g_t, \alpha-\alpha_t\right\rangle+\frac{L_G}{2}\left\|\alpha-\alpha_t\right\|_2^2\\
&\le G(\alpha)+L_G\left\|\alpha-\alpha_t\right\|_2^2 .
\end{split}
\end{equation}
Then we have

\begin{equation*}
    \begin{split}
        G(\alpha_{t+1}) &\le G\left(\alpha_t\right)+\left\langle \widehat {g_t}, \alpha_{t+1}-\alpha_t\right\rangle+\left\langle g_t-\widehat{g_t}, \alpha_{t+1}-\alpha_t\right\rangle   +\frac{L_G}{2}\left\|\alpha_{t+1}-\alpha_t\right\|_2^2\\
        & \le G\left(\alpha_t\right)+\left\langle \widehat {g_t}, \alpha_{t+1}-\alpha_t\right\rangle+\frac{1}{2L_G}\|g_t-\widehat{g_t}\|_2^2+\frac{L_G}{2}\|\alpha_{t+1}-\alpha_t\|_2^2+\frac{L_G}{2}\left\|\alpha_{t+1}-\alpha_t\right\|_2^2\\
        &=G\left(\alpha_t\right)+\left\langle \widehat {g_t}, \alpha_{t+1}-\alpha_t\right\rangle+\frac{1}{2L_G}\|g_t-\widehat{g_t}\|_2^2+L_G\|\alpha_{t+1}-\alpha_t\|_2^2\\
        &=\min_{\alpha \in W}G(\alpha_t)+\langle \widehat{g}_t,\alpha-\alpha_t \rangle+L_G\|\alpha-\alpha_t\|_2^2 +\frac{1}{2L_G}\|g_t-\widehat{g_t}\|_2^2\\
        &=\min_{\alpha \in W}G(\alpha_t)+\langle g_t,\alpha-\alpha_t \rangle+L_G\|\alpha-\alpha_t\|_2^2 +\langle \widehat{g_t}-g_t, \alpha-\alpha_t\rangle +\frac{1}{2L_G}\|g_t-\widehat{g_t}\|_2^2\\
        &\le \min_{\alpha \in W} G(\alpha_t)+
        \langle g_t,\alpha-\alpha_t \rangle+L_G\|\alpha-\alpha_t\|_2^2+\frac{L_G}{2}\|\alpha-\alpha_t\|_2^2+\frac{1}{2L_G}\|g_t-\widehat{g_t}\|_2^2+\frac{1}{2L_G}\|g_t-\widehat{g_t}\|_2^2 \\
        &=\min_{\alpha \in W} G(\alpha_t)+
        \langle g_t,\alpha-\alpha_t \rangle+\frac{3L_G}{2}\|\alpha-\alpha_t\|_2^2+\frac{1}{L_G}\|g_t-\widehat{g_t}\|_2^2\\
        &= G(\alpha_t)-\frac{1}{6L_G} \|g_t\|_2^2+\frac{1}{L_G}\|g_t-\widehat{g_t}\|_2^2
    \end{split}
\end{equation*}
where the first inequality comes from \eqref{eq:L_G_inequality}, the second inequality comes from Cauchy–Schwarz inequality, the second equality holds because the definition of $\alpha_{t+1}$, the  third inequality holds because of 
Cauchy–Schwarz inequality again, and the fifth equality holds by taking $\alpha=\alpha_t-\frac{1}{3L_g}g_t$. Telescoping over $t$, we have

$$
\frac{\sum_{t=0}^{T-1}\|g_t\|_2^2}{T}\le \frac{6L_G}{T}(G(\alpha_0)-G(\alpha_T))+6\sum_{t=0}^{T-1}\frac{\|g_t-\widehat{g}_t\|_2^2}{T}
$$
Since
$$
\mathbb{E}\left[\left\|g_t-\widehat{g_t}\right\|_2^2\right] \leq \frac{\sigma_{\xi}}{r}
$$
Then we have
$$
\mathbb{E}\|g_\text{out}\|_2^2\le \frac{6L_G}{T}(G(\alpha_0)-\mathbb{E}G(\alpha_T))+6\frac{\sigma_{\xi}}{r}
$$
Let $T=\epsilon^{-2},r=\epsilon^{-2}$, then 
$$
\mathbb{E}\|g_\text{out}\|_2^2=\mathcal{O}(\epsilon^2)
$$
and then 
$$
\mathbb{E}\|g_\text{out}\|_2=\mathcal{O}(\epsilon)
$$

\end{proof}
\color{black}

\subsection{Proof of Theorem \ref{thm:consistency}}
\label{appendix:consistency_proof}

\begin{assumption}
\label{assum:bdro}
{(Assumption 3.1 in \citep{shapiro2023bayesian})}
~
\begin{enumerate}
    \item[(1)] The set $\Theta$ is convex compact with nonempty interior.
    \item[(2)] $\ln \mu_{0}(\theta)$ is bounded on $\Theta$, i.e., there are constants $c_1>c_2>0$ such that $c_1 \geq \mu_0(\theta) \geq c_2$ for all $\theta \in \Theta$.
    \item[(3)] $P^*(\zeta)>0$ for any $\zeta\in\Xi$. 
    \item[(4)] $P_{\theta}(\zeta )>0$, and hence $\mu_N(\theta )>0$, for all $\xi \in \Xi$ and $\theta \in \Theta$.
    \item[(5)] $P_{\zeta}(\xi )$ is continuous in $\theta \in \Theta$.
    \item[(6)] $\ln P_{\theta}(\zeta ), \theta \in \Theta$, is dominated by an integrable (w.r.t. $P_*$ ) function.
\end{enumerate}
\end{assumption}

Assumption~\ref{assum:bdro} (1), (2) are used to guarantee the uniform convergence of posterior. Assumption~\ref{assum:bdro} (3), (4) require that all data points has positive probability to be sampled under the prior and posterior. Assumption~\ref{assum:bdro} (5), (6) are used to exchange the order of limit and integral.

With Assumption~\ref{assum:bdro}, we are now ready to prove Theorem~\ref{thm:consistency}. Define a function $\psi(\theta)=\mathbb{E}_{P^*}[\ln P_\theta (\xi)] $ and let $\Theta^*:=\{\theta' \in \Theta: \psi(\theta')=\inf_{\theta \in \Theta} \psi(\theta)\}.$
For $\epsilon>0$, define sets
\[
V_\epsilon:=\left\{\theta \in \Theta: \psi\left(\theta^*\right)-\psi(\theta) \geq \epsilon\right\}, U_\epsilon:=\Theta \backslash V_\epsilon=\left\{\theta \in \Theta: \psi\left(\theta^*\right)-\psi(\theta)<\epsilon\right\}.
\]
First we need to show two intermediate lemmas.
\begin{lemma}{(Lemma 3.1. \citep{shapiro2023bayesian})}
Suppose that Assumption \ref{assum:bdro} holds. Then for $0<\epsilon_2<\epsilon_1<\epsilon_0$, it follows that w.p. 1 for $N$ large enough
\[
\sup _{\theta \in V_{\epsilon_0}} \mu_N(\theta) \leq \kappa(\epsilon_2)^{-1} e^{-N(\epsilon_1-\epsilon_2)},
\]
where $V_{\epsilon_0}$ and $U_{\epsilon_0}$ are defined in (3.2), and $\kappa({\epsilon_2}):=\int_{U_{\epsilon_2}} d \theta$.
\end{lemma}

\begin{lemma}
    Suppose that Assumption \ref{assum:bdro} holds. $\forall \delta >0$, $\exists \epsilon>0$ such that $d(\theta,\Theta^*)<\delta$ for all $ \theta \in U_\epsilon $.
\end{lemma}
\begin{proof}
    We prove this lemma by contradiction. Suppose that $\exists \delta_0 >0$ such that $\forall \epsilon>0 $, there exists $\theta \in \Theta $ satisfying  $\psi\left(\theta^*\right)-\psi(\theta)<\epsilon$ and  $d(\theta,\Theta^*) \ge\delta_0$.

    Choose $\epsilon=\frac{1}{n}$ and then get a sequence $\{\theta_n\}_{n=1}^\infty$. As $\Theta$ is compact, there exists a subsequence of $\{\theta_n\}_{n=1}^\infty$ that converge to a point $\theta' \in \Theta$ satisfying $d(\theta',\Theta^*) \ge\delta_0.$  As $\psi$ is continuous, $\psi(\theta')=\psi(\theta^*)$. Contradiction!
\end{proof}

Then we can prove Theorem \ref{thm:consistency}
\begin{proof}
For any $\delta>0$, we can  choose $\epsilon_0$ such that $d(\theta,\Theta^*)\le \delta $ for $\theta \in U_{\epsilon_0}$. Then we have
\begin{equation*}
\begin{split}
     &|\rho_{\theta\sim\mu_N}(C(\alpha,\theta))-C(\alpha,\theta^*)|\\
  & =  |\max _{\xi: \xi  \in \mathcal{U}\left(\mu_N\right)} \int_{\theta \in \Theta} \xi(\theta) \mu_N(\theta)[C(\alpha,\theta) -C(\alpha,\theta^*)]d\theta|\\
    &\le \max _{\xi: \xi \in \mathcal{U}\left(\mu_N\right)} 
    \int_{U_{\epsilon_0}} \xi(\theta) \mu_N(\theta)|C(\alpha,\theta) -C(\alpha,\theta^*)| d\theta 
    +\max _{\xi: \xi  \in \mathcal{U}\left(\mu_N\right)} \int_{V_{\epsilon_0}} \xi(\theta) \mu_N(\theta)|C(\alpha,\theta) -C(\alpha,\theta^*)|d\theta \\
    &\le \sup_{\|\theta- \theta^*\|\le \delta}|C(\alpha,\theta) -C(\alpha,\theta^*)|+2\sup_{\alpha\in W,\theta\in \Theta}|C(\alpha,\theta)|\max _{\xi: \xi \mu_N \in \mathcal{U}\left(\mu_N\right)} \int_{V_\epsilon} \xi(\theta) \mu_N(\theta)d\theta 
 \end{split} 
\end{equation*}

By Holder's Inequality, we have
\begin{equation*}
    \begin{split}
        \int_{V_\epsilon} \xi(\theta) \mu_N(\theta) d\theta &=        \int_{V_\epsilon} \xi(\theta) \mu_N(\theta)^{1/q}\mu_N(\theta)^{1/p} d\theta\\
        &\le \left[\int_{V_\epsilon} \xi(\theta)^q \mu_N(\theta)d\theta\right]^{1/q} \left[\int_{V_\epsilon}  \mu_N(\theta)d\theta\right]^{1/p}\\
        &\le B_q\kappa(\epsilon_2)^{-1/p} e^{-N(\epsilon_1-\epsilon_2)/p}\cdot Vol(\Theta)^{1/p}
    \end{split}
\end{equation*}

Thus
\begin{align*}
   |\rho_{\theta\sim\mu_N}(C(\alpha,\theta))-C(\alpha,\theta^*)|
\le \delta L_\theta+2B_q\kappa(\epsilon_2)^{-1/p} e^{-N(\epsilon_1-\epsilon_2)/p}Vol(\Theta)^{1/p}\sup_{\alpha\in W,\theta\in \Theta}|C(\alpha,\theta)|, 
\end{align*}
where the inequality holds because $C(\alpha,\theta)$ is $B-$ Lipschitz continuous w.r.t. $\theta$. 
This implies $D_N\to 0$ as $N\to \infty$ since $\delta$ is arbitrary.
Then we have 
\begin{equation*}
    \begin{split}
   &C(\alpha_N^*,\theta^*)-C(\alpha^*,\theta^*)\\
   &=C(\alpha_N^*,\theta^*)-\rho_{\theta\sim\mu_N}(C(\alpha_N^*,\theta))+\rho_{\theta\sim\mu_N}(C(\alpha_N^*,\theta))-\rho_{\theta\sim\mu_N}(C(\alpha^*,\theta))+\rho_{\theta\sim\mu_N}(C(\alpha^*,\theta))-C(\alpha^*,\theta^*)\\
    &\le  2\delta B+4B_q\kappa(\epsilon_2)^{-1/p} e^{-N(\epsilon_1-\epsilon_2)/p}Vol(\Theta)^{1/p}\sup_{\alpha\in W,\theta\in \Theta}|C(\alpha,\theta)|,
    \end{split}
\end{equation*}
Let $N\to \infty$ and recall that $\delta$ is arbitary, we get the result.

% \color{blue}
Define
\begin{align*}
    E_N:=\sup_{\alpha\in W}\left\|\nabla_\alpha \rho_{\theta\sim\mu_N}(C(\alpha,\theta))-\nabla_\alpha C(\alpha,\theta^*)\right\|_2.
\end{align*}
Similar to $D_N$, we have
\begin{align*}
    \begin{aligned}
        &\left\|\nabla_\alpha \rho_{\theta\sim\mu_N}(C(\alpha,\theta))-\nabla_\alpha C(\alpha,\theta^*)\right\|_2\\
        &=\left\|\int_{\Theta} \xi_\alpha^*(\theta) \mu_N(\theta)[\nabla_\alpha C(\alpha,\theta)-\nabla_\alpha C(\alpha,\theta^*)]d\theta\right\|_2\\
        &\le  \int_{U_{\epsilon_0}} \xi_\alpha^*(\theta) \mu_N(\theta)\|\nabla_\alpha C(\alpha,\theta) -\nabla_\alpha C(\alpha,\theta^*)\|_2 d\theta 
    +\int_{V_{\epsilon_0}} \xi_\alpha^*(\theta) \mu_N(\theta)\|\nabla_\alpha C(\alpha,\theta) -\nabla_\alpha C(\alpha,\theta^*)\|_2d\theta \\
    &\le \delta L_{\theta,2}+2B_q\kappa(\epsilon_2)^{-1/p} e^{-N(\epsilon_1-\epsilon_2)/p}Vol(\Theta)^{1/p}\sup_{\alpha\in W,\theta\in \Theta}\|\nabla_\alpha C(\alpha,\theta)\|,
    \end{aligned}
\end{align*}
which implies $E_N\to 0$ as $N\to \infty$ since $\delta$ is arbitrary.

\end{proof}
\subsection{Proof of Theorem \ref{thm:two_episode} and Corollary \ref{thm:global convergence}}

\label{appendix:proof_of_two_episodes}

\begin{assumption}
\label{assum:local_invertible_parametrization}{(Assumption 5.11 in \citep{zhang2021convergence})}
 For policy parameterization $\pi_\alpha$, $\alpha$ overparametrizes the set of policies in the following sense. (i). For any $\alpha$ and $\lambda(\alpha)$ under the true environment $P_{\theta^*}$, there exist (relative) neighourhoods $\alpha \in \mathcal{U}_\alpha \subset W$ and $\lambda(\alpha) \in \mathcal{V}_{\lambda(\alpha)} \subset \lambda(W,\theta^*)$ s.t. $\left(\lambda|_{\mathcal{U}_\alpha}\right)(\cdot)$ forms a bijection between $\mathcal{U}_\alpha$ and $\mathcal{V}_{\lambda(\alpha)}$, where $\left(\lambda \mid \mathcal{U}_\alpha\right)(\cdot)$ is the confinement of $\lambda$ onto $\mathcal{U}_\alpha$. We assume $\left(\lambda \mid \mathcal{U}_\alpha\right)^{-1}(\cdot)$ is $\ell_\alpha$-Lipschitz continuous and $\left(\lambda \mid \mathcal{U}_\alpha\right)(\cdot)$ is $L_\lambda$-Lipschitz smooth  for any $\alpha$ . (ii). Let $\pi_{\alpha^*}$ be the optimal policy under the true environment. Assume there exists $\bar{\epsilon}$ small enough, s.t. $(1-\epsilon) \lambda(\alpha)+\epsilon \lambda\left(\alpha^*\right) \in$ $\mathcal{V}_{\lambda(\alpha)}$ for $\forall \epsilon \leq \bar{\epsilon}, \forall \alpha$.
\end{assumption}

% \color{red}
% \begin{proof}
%     We assume that each $G_i(\alpha):=\rho_{\theta\sim \mu_i}(C(\alpha,\theta))$ has $L_{G,i}$ Lipschitz continuous gradient and define the gradientgap term
%     \[
%     y_{i,j}:=\mathbb{E}[\nabla_\alpha G_i(\alpha_{i,j})]
%     \]
%     By Theorem \ref{thm:converge_alg}, we have
%     $$
%     y_{i+1,t_{i+1}}\le (1-\epsilon)^{t_{i+1}}y_{i+1,0}+2L_{G,i+1}\ell_\alpha D_\lambda^2\epsilon+\frac{\sigma_\xi}{r_{i+1}L_{G,i+1}\epsilon}.
%     $$

% Then we connect $i+1$-th episode with the previous one. Notice that it holds for any $\alpha$ that 
% \begin{equation*}
% \begin{split}
%        &G_{i+1}(\alpha)-G_{i+1}^*\\
%        &=G_{i+1}(\alpha)-G_i(\alpha)+G_i(\alpha)-G_i^*\\
%        &+G_i^*-G_i(\alpha_{i+1}^*)+G_i(\alpha_{i+1}^*)-G_{i+1}(\alpha_{i+1}^*)\\
%        &\le 2D_i+ 2D_{i+1}+G_i(\alpha)-G_i^*,
% \end{split}
% \end{equation*}
% which implies 
% \begin{equation*}
%     y_{i+1,0}\le y_{i,t_i}+2(D_i+ D_{i+1}).
% \end{equation*}

% Thus we have
% \begin{equation*}
%     y_{i+1,t_{i+1}}\le (1-\epsilon)^{t_{i+1}}y_{i,t_i}+(1-\epsilon)^{t_{i+1}}2(D_i+D_{i+1})+2L_{G,i+1}\ell_\alpha D_\lambda^2\epsilon+\frac{\sigma_\xi}{r_{i+1}L_{G,i+1}\epsilon}.
% \end{equation*}

% Note that $(1-\epsilon)^{\epsilon^{-1}}\le 1/2,\forall \epsilon\le 1$. By choosing $t_{i+1}\ge \mathcal{O}(\epsilon^{-1}\log(\frac{D_i+D_{i+1}}{\epsilon}))$ and $r_{i+1}=\Theta(\epsilon^{-2}/L_{G,i+1})$, we can keep an error bound $\mathcal{O}(\epsilon)$ for each episode.
% \end{proof}
% \color{blue}

Under the true environment $P_{\theta^*}$, the set of all occupancy measures is a convex set, and there is an bijection between all policies and all occupancy measures. More discussions can be found in Section 5.2 in \citep{zhang2021convergence}. Based on this observation, we have $\min_\pi  F(\lambda^\pi,\theta^*)=\min_{\lambda} F(\lambda,\theta^*)$, which turns the non-convex policy optimization problem into a convex optimization problem over occupancy measure. Then any stationary point is also globally optimal, which is shown in the following lemma. 
\begin{lemma}
\label{lem:gradient_bound_function}
Assume that Assumption \ref{assum:local_invertible_parametrization} holds.  Then $C(\bar{\alpha},\theta^*)-C(\alpha^*,\theta^*)\le \mathcal{O}(\|\nabla_\alpha C(\bar{\alpha},\theta^*)\|_2),\forall \bar{\alpha}\in W$.
\end{lemma}
\begin{proof}
Notice that
\begin{align*}
        \partial (F\circ\lambda)(\alpha)=\nabla_\alpha\lambda(\alpha)^\top \partial F(\lambda),\forall \alpha\in W.
\end{align*}
So there exists $\bar{w}\in\partial F(\bar{\lambda},\theta^*)$ s.t. $\nabla_\alpha C(\bar{\alpha},\theta^*)=\nabla_\alpha\lambda(\bar{\alpha})^\top \bar{w}.$ Then for any $\lambda(\alpha)\in\mathcal{V}_{\lambda(\bar{\alpha})}$, we have
\begin{equation}
\begin{aligned}
        &\langle\bar{w},\lambda-\bar{\lambda}\rangle\\
    &=\langle\bar{w},\nabla_\alpha\lambda(\bar{\alpha})(\alpha-\bar{\alpha})\rangle+\langle\bar{w},  \lambda-\bar{\lambda}-\nabla_\alpha\lambda(\bar{\alpha})(\alpha-\bar{\alpha})\rangle\\
    &=I_1+I_2
\end{aligned}
\end{equation}
For the first term, we have
\begin{align*}
    I_1=\langle \nabla_\alpha C(\bar{\alpha},\theta^*) ,\alpha-\bar{\alpha}\rangle\ge -\|\nabla_\alpha C(\bar{\alpha},\theta^*)\|_2 \|\alpha-\bar{\alpha}\|_2\ge -\|\nabla_\alpha C(\bar{\alpha},\theta^*)\|_2 \ell_\alpha\|\lambda-\bar{\lambda}\|_2.
\end{align*}
For the second term, we have 
\begin{align*}
    I_2\ge -\|\bar{w}\|_2\cdot \frac{L_\lambda}{2}\|\alpha-\bar{\alpha}\|^2_2\ge  -\frac{L_\lambda \ell_\alpha^2}{2} \|\bar{w}\|_2 \|\lambda-\bar{\lambda}\|_2^2
\end{align*}

Then we  have
\begin{align*}
   \langle\bar{w},\lambda-\bar{\lambda}\rangle\ge   -\|\nabla_\alpha C(\bar{\alpha},\theta^*)\|_2 \ell_\alpha\|\lambda-\bar{\lambda}\|_2-\frac{L_\lambda \ell_\alpha^2}{2} \|\bar{w}\|_2 \|\lambda-\bar{\lambda}\|_2^2.
\end{align*}
Replace $\lambda$ by $(1-\epsilon) \lambda(\bar{\alpha})+\epsilon \lambda\left(\alpha^*\right)$ for any $\epsilon\in(0,\bar{\epsilon}]$ and then it holds
\begin{align*}
    \epsilon\langle\bar{w},\lambda\left(\alpha^*\right)-\lambda(\bar{\alpha})\rangle\ge  -\epsilon\|\nabla_\alpha C(\bar{\alpha},\theta^*)\|_2 \ell_\alpha\|\lambda\left(\alpha^*\right)-\lambda(\bar{\alpha})\|_2-\frac{L_\lambda \epsilon^2\ell_\alpha^2}{2} \|\bar{w}\|_2 \|\lambda\left(\alpha^*\right)-\lambda(\bar{\alpha})\|_2^2.
\end{align*}

Divide both sides by $\epsilon$ and let $\epsilon\to 0$, we have
$$
\langle\bar{w},\lambda\left(\alpha^*\right)-\lambda(\bar{\alpha})\rangle\ge -\|\nabla_\alpha C(\bar{\alpha},\theta^*)\|_2 \ell_\alpha\|\lambda\left(\alpha^*\right)-\lambda(\bar{\alpha})\|_2.
$$

Finally,
$$
C(\bar{\lambda},\theta^*)-C(\lambda^*,\theta^*)\le-\langle \bar{w},\lambda\left(\alpha^*\right)-\lambda(\bar{\alpha})\rangle\le  \ell_\alpha \|\nabla_\alpha C(\bar{\alpha},\theta^*)\|_2\|\lambda\left(\alpha^*\right)-\lambda(\bar{\alpha})\|_2\le  \ell_\alpha D_\lambda \|\nabla_\alpha C(\bar{\alpha},\theta^*)\|_2,
$$
where $D_\lambda=\sup_{\alpha,\alpha'\in W}\|\lambda(\alpha)-\lambda(\alpha')\|_2$.
\end{proof}

Then we can turn to prove Theorem \ref{thm:two_episode}.
\begin{proof}
If $\mathbb{E}\left[\frac{\sum_{j=0}^{t_i-1}\|\nabla G_i(\alpha_{i,j})\|}{t_i}\right]\le\epsilon$, choose $\alpha_{i+1,0}$ uniformly from $\alpha_{i,0},\dots,\alpha_{i,t_i-1}$.
Then $\mathbb{E}\|\nabla G_i(\alpha_{i+1,0})\|_2\le \epsilon $ and then
$$
\mathbb{E}\|\nabla C(\alpha_{i+1,0},\theta^*)\|_2\le \epsilon+E_i.
$$
By Lemma \ref{lem:gradient_bound_function} we have
$$
\mathbb{E}C(\alpha_{i+1,0},\theta^*)-C(\alpha^*,\theta^*)\le (\epsilon+E_i)\ell_\alpha D_\lambda,
$$
and then
$$
\mathbb{E}G_{i+1}(\alpha_{i+1,0})-G_{i+1}(\alpha_{i+1}^*)\le (\epsilon+E_i)\ell_\alpha D_\lambda+2D_{i+1}.$$
By Theorem \ref{thm:converge_alg} we have
    $$
\mathbb{E}\left[\frac{\sum_{j=0}^{t_{i+1}-1}\|\nabla G_{i+1}(\alpha_{i+1,j})\|_2^2}{t_{i+1}}\right]\le \frac{6L_{G,i+1}}{t_{i+1}}[(\epsilon+E_i)\ell_\alpha D_\lambda+2D_{i+1}]+6\frac{\sigma_{\xi}}{r}
    $$
    Then we can choose
    $$
 t_{i+1}=12L_{G,i+1}[(\epsilon+E_i)\ell_\alpha D_\lambda+2D_{i+1}]\epsilon^{-2}
    $$
    $$
r=12\sigma_{\xi}\epsilon^{-2}
    $$
    to make
    $$
\mathbb{E}\left[\frac{\sum_{j=0}^{t_{i+1}-1}\|\nabla G_{i+1}(\alpha_{i+1,j})\|_2^2}{t_{i+1}}\right]\le \epsilon^2.
    $$
    Then by Jensen's inequality we have
    $$
\mathbb{E}\left[ \frac{\sum_{j=0}^{t_{i+1}-1}\|\nabla G_{i+1}(\alpha_{i+1,j})\|_2}{t_{i+1}}\right]\le\epsilon.
    $$
When $\mathbb{E}\|\nabla G_N(\alpha_{\text{out}})\|_2\le \epsilon$, we have  $\|\nabla C(\alpha_{\text{out}},\theta^*)\|_2\le \epsilon+E_N$. Then by Lemma \ref{lem:gradient_bound_function} we complete the proof of Theorem \ref{thm:two_episode}. 

As Theorem \ref{thm:consistency} shows that $E_N\to \infty$ when $N\to \infty$, then we complete the proof of Corollary  \ref{thm:global convergence}.
\end{proof}

% \color{red}
% Old proof:

% \begin{proof}
%     We assume that each $G_i(\alpha):=\rho_{\theta\sim \mu_i}(C(\alpha,\theta))$ has $L_{G,i}$ Lipschitz continuous gradient and define the gap term
%     \[
%     y_{i,j}:=\mathbb{E}[G_i(\alpha_{i,j})-G_i(\alpha_i^*)]
%     \]
%     By Theorem \ref{thm:converge_alg}, we have
%     $$
%     y_{i+1,t_{i+1}}\le (1-\epsilon)^{t_{i+1}}y_{i+1,0}+2L_{G,i+1}\ell_\alpha D_\lambda^2\epsilon+\frac{\sigma_\xi}{r_{i+1}L_{G,i+1}\epsilon}.
%     $$

% Then we connect $i+1$-th episode with the previous one. Notice that it holds for any $\alpha$ that 
% \begin{equation*}
% \begin{split}
%        &G_{i+1}(\alpha)-G_{i+1}^*\\
%        &=G_{i+1}(\alpha)-G_i(\alpha)+G_i(\alpha)-G_i^*\\
%        &+G_i^*-G_i(\alpha_{i+1}^*)+G_i(\alpha_{i+1}^*)-G_{i+1}(\alpha_{i+1}^*)\\
%        &\le 2D_i+ 2D_{i+1}+G_i(\alpha)-G_i^*,
% \end{split}
% \end{equation*}
% which implies 
% \begin{equation*}
%     y_{i+1,0}\le y_{i,t_i}+2(D_i+ D_{i+1}).
% \end{equation*}

% Thus we have
% \begin{equation*}
%     y_{i+1,t_{i+1}}\le (1-\epsilon)^{t_{i+1}}y_{i,t_i}+(1-\epsilon)^{t_{i+1}}2(D_i+D_{i+1})+2L_{G,i+1}\ell_\alpha D_\lambda^2\epsilon+\frac{\sigma_\xi}{r_{i+1}L_{G,i+1}\epsilon}.
% \end{equation*}

% Note that $(1-\epsilon)^{\epsilon^{-1}}\le 1/2,\forall \epsilon\le 1$. By choosing $t_{i+1}\ge \mathcal{O}(\epsilon^{-1}\log(\frac{D_i+D_{i+1}}{\epsilon}))$ and $r_{i+1}=\Theta(\epsilon^{-2}/L_{G,i+1})$, we can keep an error bound $\mathcal{O}(\epsilon)$ for each episode.
% \end{proof}

\color{black}

\section{Examples of Loss Function}
\label{appendix:loss_function}
% \begin{example}[Linearly Constrained MDP]
% In the safe-learning problem, which is usually formulated as a constrained MDP \citep{altman2021constrained}, the goal is to minimize the total expected discounted cost under a constraint where for some reward function $r: \mathcal{S} \times \mathcal{A} \to \mathbb{R}$, the total expected discounted reward is constrained from below. So the counterpart constrained formulation to \eqref{original_formulation} is
% \begin{align*}
%     \min_{\pi} \quad & \mathbb{E}[\sum_{t=0}^{\infty}\gamma^t c(s_t,a_t)\mid \pi, s_0 \sim \tau]\\
%     \text{s.t.} \quad & \mathbb{E}[\sum_{t=0}^{\infty}\gamma^t r(s_t,a_t)\mid \pi, s_0 \sim \tau] \geq R.
% \end{align*}
% Using Lagrangian relaxation, we can choose $F$ to be a linear function of $\lambda$, i.e., $F(\lambda, P) = \langle \lambda, c \rangle - \ell (\langle \lambda, r \rangle-R)$, where $\ell$ is the Lagrange multiplier. 
% \end{example}
\begin{example}[Risk-Averse  Constrained  MDP]
\label{examlple:risk-averse CMDP}
In safe RL problems, one usually considers a constrained MDP \citep{altman2021constrained}, where the goal is to minimize the total expected discounted cost under a risk-averse constraint. Given a random vector penalty $d$, the risk-averse constraint is  to control a risk measure of  the total expected discounted penalty. This leads to the following  constrained MDP formulation:
\begin{equation*}
\begin{aligned}
     &\min_{\pi} \  \mathbb{E}[\sum_{t=0}^{\infty}\gamma^t c(s_t,a_t)\mid \pi, s_0 \sim \tau] \quad \text{s.t.} \ \rho \left(\mathbb{E}[\sum_{t=0}^{\infty}\gamma^t d(s_t,a_t)\mid \pi, s_0 \sim \tau] \right)\leq D,   
\end{aligned}
\end{equation*}
where $\rho$ is a coherent risk measure, such as Conditional Value-at-Risk (CVaR)\footnote{ $\text{CVaR}(X)=\mathbb{E}[X|X\geq v_\beta(X)]$, where $ v_\beta(X)$ is a $\beta$-quantile of $X$, i.e. $\mathbb{P}(X\ge  v_\beta(X) )=1-\beta$}
Using Lagrangian relaxation, we can choose $F$ to be a convex function of $\lambda$, i.e., $F(\lambda, P) = \langle \lambda, c \rangle +\ell (\rho(\langle \lambda, d \rangle)-D)$, where $\ell$ is the Lagrange multiplier. 
\end{example}

\begin{example}[Imitation Learning]
    During imitation learning, the agent learn through some demonstrations to behave similarly to an expert. One formulation is minimize the $f-$divergence between the occupancy measure of the current policy and the target occupancy measure: 
    \begin{equation*}
        \min_{\pi} D_f(\lambda^\pi, q)=\sum_{s,a} q(s,a) f\left(\frac{\lambda^\pi(s,a)}{q(s,a)}\right)
    \end{equation*}
\end{example}
\section{Examples of Risk Envelop}
\label{appendix:example_envelop}
\begin{example}
\label{ex:CVaR}[Conditional Value at Risk] First, Value-at-risk $\operatorname{VaR}_{\beta}(X)$ is defined as the $\beta$-quantile of $X$, i.e.,
$\operatorname{VaR}_{\beta}(X):=\inf \{t: \mathbb{P}(X \leq t) \geq \beta\}$, where the confidence level $\beta \in (0,1)$. Assuming there is no probability atom at $\operatorname{VaR}_{\beta}(X)$, CVaR at confidence level $\beta$ is defined as the mean of the $\beta$-tail distribution of $X$, i.e., $\operatorname{CVaR}_{\beta}(X)=\mathbb{E}\left[X \mid X \geq \operatorname{VaR}_{\beta}(X)\right]$.
The envelope set is
\[\mathcal{U}(\mu_N)=\{\xi \in \mathcal{Z}^*: \int_{\Theta} \xi(\theta)\mu_N(\theta)d\theta=1, \xi (\theta)\in \left[0,\frac{1}{1-\beta}\right] a.s. \theta \in \Theta\}\]
\end{example}

\begin{example}
(Mean-Upper-Semideviation of Order $\boldsymbol{p})$. For $\mathcal{Z}:=\mathcal{L}_p(\Theta, \mathcal{F}, \mu_N)$ and $\mathcal{Z}^*:=\mathcal{L}_q(\Theta, \mathcal{F}, \mu_N)$, with $p \in[1,+\infty)$, $c \in [0,1]$ and $\mathcal{F}$ to be a $\sigma$-field on $\Theta$, consider
\[
\rho(Z):=\mathbb{E}[Z]+c\left(\mathbb{E}\left[[Z-\mathbb{E}[Z]]_{+}^p\right]\right)^{1 / p},
\]
where $[a]_+^p=\max\{0,a\}^p$. 
Then the envelope set is 
\[
\mathcal{U}(\mu_N)=\left\{\xi^{\prime} \in \mathcal{Z}^*: \xi^{\prime}=1+\xi-\mathbb{E}[\zeta],\|\xi\|_q \leq c,\xi \succeq 0\}\right\} .
\]
\end{example}

More examples can be found in Section 6.3.2\citep{shapiro2021lectures}.

\section{Policy Gradient for MDP with CVaR Risk Measure : A Special Case Study}
\label{appendix:CVaRCase}
Here we offer an example of gradient estimator with a common coherent risk measure  Conditional Value at Risk(CVaR), the definition of which can be found in Example \ref{ex:CVaR}.
For the considered CVaR risk functional, \citep{hong2009simulating} shows that the gradient of the CVaR risk functional can be expressed as
\begin{align*}
    \nabla \operatorname{CVaR}_\beta(X(\alpha))=\mathbb{E}[\nabla X(\alpha)| X(\alpha)\ge v_\beta(\alpha)]
\end{align*}
where $v_\beta=v_\beta(\alpha):=\operatorname{VaR}_{\beta}(X(\alpha))$ for a random parameterized variable $X(\alpha)$ satisfying Assumption \ref{assum:sensitivity}. Unless otherwise specified, the derivative is assumed to be taken w.r.t. $\alpha$.

\begin{assumption}{(Assumption 1, 2, 3 \citep{hong2009simulating})}
\label{assum:sensitivity}
(i) There exists a random variable $L$ with $\mathbb{E}(K)<\infty$ such that $\left|X\left(\alpha_2\right)-X\left(\alpha_1\right)\right| \leq K \left\|\alpha_2-\alpha_1\right\|_{2}$ for all $\alpha_1, \alpha_2 \in W$, and $\nabla_\alpha X(\alpha) $ exists almost surely for all $\alpha \in W$.

(ii) VaR function $v_\beta(\alpha)$ is differentiable for any $\alpha \in W$.

(iii) For any $\alpha \in W, \mathbb{P}\left(X(\alpha)=v_\beta(\alpha)\right)=0$.
\end{assumption}

Assumption~\ref{assum:sensitivity} (i) is commonly used in path-wise derivative estimation; (ii) shows that VaR function is locally Lipschitz; (iii) requires that there is no probability atom at $VaR(X)$ and implies that $\mathbb{P}(X(\alpha)\ge v_\beta(\alpha))=1-\beta$. 

\begin{theorem}\label{thm: policy gradient_CVaR}
Suppose that Assumption~\ref{assum:sensitivity} holds. Then, for any $\alpha \in W $ and $\beta\in(0,1)$, the policy gradient to the objective function in \eqref{param_policy_optimimzation} is given by:
\begin{equation}\label{eq:true_gradient}
\begin{split}
      g(\alpha)&=\mathbb{E}_{\theta \sim \mu_N}\left[\nabla C(\alpha,\theta) \mid C(\alpha,\theta) \geq v_\beta(\alpha)\right]\\      
      &=\frac{1}{1-\beta}\mathbb{E}_{\theta \sim \mu_N}\left[\nabla C(\alpha,\theta) \mathbbm{1}_{\{C(\alpha,\theta)\ge v_\beta\}}  \right]
\end{split}
\end{equation}
where $\mathbbm{1}_{\{\cdot\}}$ is the indicator function.
\end{theorem}

If we apply Theorem \ref{thm:gradient_coherent} to CVaR, we will get the same result as Theorem\ref{thm: policy gradient_CVaR}. To compute the gradient $g(\alpha)$, we require the cumulative value $C(\alpha,\theta)$ of policy $\pi_{\alpha}$ and its gradient $\nabla C(\alpha,\theta)$, value-at-risk $v_{\beta}$, as well as the evaluation of the expectation taken w.r.t. the posterior distribution $\mu_N$. Here we show how to use zeroth-order method instead of variational approach to estimate $\nabla_\alpha C(\alpha,\theta)$. Since there is no closed-form expression for the expectation, we estimate the gradient $g(\alpha)$ with samples $\{\theta^i\}_{i=1}^n$ generated from $\mu_N$. We construct the gradient estimator as follows:
\begin{equation}
\label{eq:gradient_estimator}
\widehat{g}(\alpha)=\frac{1}{n(1-\beta)}\sum_{i=1}^n  \widehat{\nabla C}(\alpha,\theta^i ) \mathbbm{1}_{\{\widehat{C}(\alpha,\theta^i)\geq \widehat{v}_\beta\}}.
\end{equation} 

For a fixed $\alpha$ and $\theta^i$, we first estimate the occupancy measure $\lambda^i$ by making a truncation of horizon $K$ in \eqref{eq:def_occu} with error 
\begin{equation}
\label{eq:eps_lamda}
   \|\widehat \lambda^i -\lambda^i\|_\infty\le \epsilon_\lambda:=\gamma^K/(1-\gamma) 
\end{equation}
 for some $K>0$. The cumulative value with the truncated occupancy measure $\widehat{\lambda}^i$ is denoted by $\widehat{C}(\alpha,\theta^i)=F(\widehat \lambda,P_{\theta^i})$. The value-at-risk estimate is $\widehat{v}_\beta:=\widehat{C}(\alpha,\theta)_{\lceil n\beta \rceil :n}$, where $\widehat{C}(\alpha,\theta)_{\lceil n\beta \rceil :n}$ is the $\lceil n\beta \rceil$-th smallest quantity in $\{\widehat{C}(\alpha,\theta^i)\}_{i=1}^n$.

% The construction of $\widehat{\nabla C}$ is then based on the zeroth-order method. 
Here we adopt the Gaussian smoothing approach of estimating gradients from function evaluations \citep{nesterov2017random,balasubramanian2022zeroth}. When there is no oracle to the first-order information or it is not efficient to calculate the gradient directly, Gaussian smoothing approach is a useful technique in zeroth-order method. Compared with finite difference method, Gaussian smoothing approach requires weaker smoothness condition of objective function. For a fixed $\alpha$ and $\theta^i$, generate $\{u^{i,j}\}_{j=1}^{m_i}$, where $u^{i, j} \sim \mathcal{N}\left(0, I_d\right)$. Then $\widehat{\nabla C}$ can be constructed as: 

\vspace{-0.3cm}
{\small
\begin{align}\label{eq:gradient_sample}
    \widehat{\nabla C}(\alpha,\theta^i) =\frac{1}{m_i} \sum_{j=1}^{m_i} \frac{\widehat{C}\left(\alpha+\nu u^{i, j}, \theta^i\right)-\widehat{C}\left(\alpha,\theta^i\right)}{\nu} u^{i, j}
\end{align}
}

\vspace{-0.1cm}
where $\nu>0$ is the smoothing parameter.

For ease of notation, let $\widehat{G}(\alpha)$ denote the sample estimate of $\rho_{\theta\sim\mu_N}(C(\alpha,\theta))$. We use the following gradient descent step in the $t$-th iteration:

\vspace{-0.3cm}
{\small
\begin{align}
\alpha_{t+1}&=\operatorname{\arg\min}_{\alpha \in W}\widehat{G}(\alpha_t)+\langle \widehat{g}(\alpha_t),\alpha-\alpha_t \rangle+\frac{\eta_t}{2}\|\alpha-\alpha_t\|^2 \nonumber
\\
&=\operatorname{Proj}_{W}\left(\alpha_t-\frac{1}{\eta_t}\widehat{g}(\alpha_t)\right)
\end{align}
}

\vspace{-0.1cm}
where $\eta_t$ is the stepsize and $\operatorname{Proj}_{W}(x)= \operatorname{\arg\min}_{y\in W}\|y-x\|_2^2$ projects $x$ into the parameter space $W$. We summarize the full algorithm in Algorithm~\ref{alg:sgd_CVaR}.

\begin{algorithm}
\caption{BR-PG: Bayesian Risk Policy Gradient for CVaR}
\label{alg:sgd_CVaR}
\begin{algorithmic}

\State \textbf{input}: initial $\alpha_0$, data $\zeta^{(N)}$ of size $N$, prior distribution $\mu_0(\theta)$, iteration number $T$, truncation horizon $K$;
\State calculate the posterior $\mu_N(\theta)=\frac{P_{\theta}(\zeta^{(N)})\mu_0(\theta)}{\int_{\theta'} P_{\theta'}(\zeta^{(N)})\mu_0(\theta')}$;
\For{ $t =0$ to $T-1$}
\State sample $\{\theta_t^i\}_{i=1}^n$ from $\mu_N(\theta)$;
\For{ $i =1$ to $n$}
\State calculate $\widehat{\lambda}^i_t$ using the truncation of horizon $K$ specified in \eqref{eq:def_occu};
\State calculate $\widehat{C}(\alpha_t,\theta_t^i):=F(\widehat{\lambda}_t^i,P_{\theta_t^i})$;
\State generate $\{u^{i,j}\}_{j=1}^{m_i}$, where $u^{i, j} \sim \mathcal{N}\left(0, I_d\right)$;
\State calculate $
\widehat{\nabla C}(\alpha_t,\theta_t^i)$ by \eqref{eq:gradient_sample};
\EndFor
    \State calculate $\widehat{v}_\beta(\alpha_t):=\widehat{C}(\alpha_t,\theta_t^i)_{\lceil n\beta \rceil :n}$.
    \State calculate $\widehat{g}(\alpha_t)$ by \eqref{eq:gradient_estimator};
    \State update $\alpha_{t+1}$ by \eqref{eq:gradient_descent_step}.
\EndFor
\State \textbf{output}: $\alpha_T$.
\end{algorithmic}
\end{algorithm}
\subsection{Convergence Analysis for CVaR Risk Measure}\label{sec:CvaR case}
Here we only show the estimation error of the policy gradient. To get a finite-step convergence result similar to Theorem \ref{thm:converge_alg}, we only need to substitute $\mathcal{O}(r^{-1/4})$ in Theorem \ref{thm:converge_alg} with $\mathcal{O}(R^{1/2})$, where $R^2=    \mathcal{O}\left(dn^{-1}+\epsilon_\lambda+\frac{d\epsilon_\lambda^2}{\nu^2}+\frac{d+\nu^2d^3}{m}\right)
$ is the bound for $\mathbb{E}\|[g-\widehat{g}]\|_2^2] $ in Theorem \ref{thm:error_estimator}.

Here we still adopt the Assumption \ref{assum:Lipschitz_F_C} about the smoothness for the considered loss functions, which are commonly used in gradient descent analysis.
The error bound for the zeroth-order estimation for $\nabla C$ is then shown in the next lemma.
\begin{lemma}
\label{lem:bound for nabla C}
Suppose Assumption \ref{assum:sensitivity} and Assumption \ref{assum:Lipschitz_F_C} hold. Then we have for each $i \in [n]$
\begin{equation}
    \begin{split}
        &\mathbb{E}\|\widehat{\nabla{ C}}(\alpha,\theta_i)-\nabla C(\alpha,\theta_i)\|_2^2\le \frac{8d}{\nu^2}L_{F,\infty}^2\epsilon_\lambda^2 \\
         &+\frac{8(d+5)B^2}{m_i}+\frac{2\nu^2L_{C,2}^2(d+6)^3}{m_i},
    \end{split}
\end{equation}
where $L_{F,\infty},L_{C,2},B$  are constants in Assumption \ref{assum:Lipschitz_F_C}, $\epsilon_\lambda$ is the truncation error defined in \eqref{eq:eps_lamda}, $d$ is the dimension of the policy parameter $\alpha$, $m_i$ is the number of samples used to construct the zeroth-order estimator in \eqref{eq:gradient_sample}.
\end{lemma}

% \begin{proof}
% Here is a proof sketch and more details can be found in Section \ref{appendix:proof}.
% Define
% \[
% \nabla\tilde{C}(\alpha,\theta^i) =\frac{1}{m_i} \sum_{j=1}^{m_i} \frac{C\left(\alpha+\nu u_{i, j}, \theta_i\right)-C\left(\alpha,\theta_i\right)}{\nu} u_{i, j},
% \]   
% and  notice that
% \[
% \mathbb{E}\|\widehat{ \nabla C}-\nabla C\|_2^2\le 2\mathbb{E}\|\widehat{\nabla C}-\nabla\tilde{ C}\|_2^2+ 2\mathbb{E}\|\nabla\tilde{ C}-\nabla C\|_2^2.
% \]
% The first term is the estimation error of $C$ and the second term is the is the estimation error of zeroth-order method. Then bound each term respectively
% \end{proof}

% \iffalse
% \begin{assumption}{(Assumptions in Theorem 3.2 \citep{zhang2020variational})}
% \label{assum:nablaC}

% (i) $\operatorname{dom} F=\mathbb{R}^{S A}$, there exists $\ell_F$ such that $\max \left\{\|\nabla F(\lambda)\|_{\infty}:\|\lambda\|_1 \leq \frac{2}{1-\gamma}\right\} \leq \ell_F$.

% (ii) $F$ is $L_F$-smooth under $L_1$ norm, i.e., $\left\|\nabla F(\lambda)-\nabla F\left(\lambda^{\prime}\right)\right\|_{\infty} \leq L_F\left\|\lambda-\lambda^{\prime}\right\|_1$.

% (iii) $F^*$ is $\left(\ell_{F^*}\right)$-Lipschitz w.r.t. the $L_{\infty}$ norm in the set \\$\left\{z:\|z\|_{\infty} \leq 2 \ell_F, F^*(z)>-\infty\right\}$.

% (iv) There exists $C$ with $\left\|\nabla_\alpha \pi(\cdot \mid s)\right\|_{\infty, 2} \leq C$, where $\nabla_\alpha \pi(\cdot \mid s)=\left[\nabla_\alpha \pi(1 \mid s), \cdots, \nabla_\alpha \pi(A \mid s)\right]$.
% \end{assumption}
% \fi

\begin{assumption}{(Assumptions 4 and 5 in \citep{hong2009simulating}}
\label{assum:sensi_2}

(1) For all $\alpha \in W$, $C(\alpha,\theta)$ is a continuous random variable with a density function $f_{C,\alpha}(y)$. Furthermore, $f_{C,\alpha}(y)$ and $g_{C,\alpha}(y):=\mathbb{E}_{\theta}[\nabla C(\alpha,\theta) \mid C(\alpha,\theta)=y]$ are continuous at $y=v_\alpha$, and $f_{C,\alpha}\left(v_\alpha\right)>0$.

(2) $\mathbb{E}_{\theta}\left[C(\alpha,\theta)^2\right]<\infty$  for all $\alpha \in W$.
\end{assumption}

% Assumption \ref{assum:sensi_2} are essential for the variance analysis for the CVaR sensitivity estimator.

Now we are ready to show the error for our gradient estimator given in \eqref{eq:gradient_estimator}.

\begin{theorem}
\label{thm:error_estimator}
Suppose that Assumption \ref{assum:sensitivity}, Assumption \ref{assum:Lipschitz_F_C} and Assumption \ref{assum:sensi_2}  hold. Also assume that the cumulative distribution function of $C(\alpha,\theta)$ w.r.t $\theta$ is $\ell_{C}-$ Lipschitz continuous for each $\alpha \in W$. Let $m_i=m ~\forall i \in [n]$. Then for each $\alpha \in W$,
\[
\mathbb{E}\|[g-\widehat{g}]\|_2^2 \leq  \mathcal{O}\left(dn^{-1}+\epsilon_\lambda+\frac{d\epsilon_\lambda^2}{\nu^2}+\frac{d+\nu^2d^3}{m}\right),
\]
where $n$ is the number of samples of $\theta$.
\end{theorem}

\begin{proof}
First recall that the true gradient and our gradient estimator are
$g=\frac{1}{1-\beta}\mathbb{E}\left[\nabla C(\alpha,\theta) \mathbbm{1}_{\{C(\alpha,\theta)\ge v_\beta\}}\right]$ and $\widehat{g}=\frac{1}{n(1-\beta)}\sum_{i=1}^n  \widehat{\nabla C}(\alpha,\theta_i ) \mathbbm{1}_{\{\widehat{C}(\alpha,\theta_i)\ge \widehat{v}_\beta\}}$.
Let
\[
\tilde{g}=\frac{1}{n(1-\beta)}\sum_{i=1}^n  \nabla C(\alpha,\theta_i ) \mathbbm{1}_{\{C(\alpha,\theta_i)\ge \tilde {v}_\beta\}} ,
\]
and
\[
\widehat{g_1}=\frac{1}{n(1-\beta)}\sum_{i=1}^n \nabla C(\alpha,\theta_i ) \mathbbm{1}_{\{\widehat{C}(\alpha,\theta_i)\ge \widehat{v}_\beta\}} ,      
\]
where $\tilde{v}_\beta:=C(\alpha,\theta_i)_{\lceil n\beta \rceil :n}$. Then we have the decomposition $g-\widehat{g}=(g-\tilde{g})+(\tilde{g}-\widehat{g_1})+(\widehat{g_1}-\widehat{g}):=R_1+R_2+R_3$. For $R_1$, it is the error in the estimation of expectation taken w.r.t. $\theta$. Suppose that Assumption \ref{assum:sensitivity} and Assumption \ref{assum:sensi_2} hold, Theorem 4.2 from \citep{hong2009simulating} shows that
\begin{equation*}
\|\mathbb{E}R_1\|_2=\|\mathbb{E}[\tilde{g}]-g\|_2=o(n^{-1/2}d^{-1/2}).
\end{equation*}
Notice that
\[
\|g-\tilde g\|_2^2\le 2\|g-\mathbb{E}\tilde g\|_2^2+2\| \mathbb{E}\tilde g- \tilde g\|_2^2.
\]
By Theorem 4.3 from \citep{hong2009simulating}, $Var(\tilde{g})=\mathcal{O}(dn^{-1})$. Thus
\begin{equation}
\label{eq:r1}
\mathbb{E}\|R_1\|_2^2=\mathcal{O}(dn^{-1}).
\end{equation}
For $R_3$, it is the error in the estimation of $C(\alpha,\theta)$. By Lemma \ref{lem:bound for nabla C}, $\mathbb{E}[\|\widehat{\nabla C}(\alpha,\theta_i )-\nabla C(\alpha,\theta_i )\|_2^2]\le\frac{8d}{\nu^2}L_{F,\infty}^2\epsilon_\lambda^2 +\frac{8(d+5)B^2}{m_i}+\frac{2\nu^2 L_{C,2}^2(d+6)^3}{m_i}$. If we choose all $m_i$ to be the same $m$, then
\begin{equation*}
\begin{split}
\mathbb{E}[\|\widehat{g_1}-\widehat{g}\|_2^2] &\le   \frac{1}{n(1-\beta)^2}\sum_{i=1}^n\|\widehat{\nabla C}(\alpha,\theta_i)-\nabla C(\alpha,\theta_i)\|_2^2\\
&\le\mathcal{O}(\frac{d\epsilon_\lambda^2}{\nu^2}+\frac{d+5}{m}+\frac{\nu^2(d+6)^3}{m}).
\end{split}
\end{equation*}
Thus
\begin{equation}
\label{eq:r3}
    \mathbb{E}[\|R_3\|_2^2]  \le\mathcal{O}(\frac{d\epsilon_\lambda^2}{\nu^2}+\frac{d+5}{m}+\frac{\nu^2(d+6)^3}{m}).
\end{equation}
Now we consider $R_2$.
Define the event $A_i=\{C(\alpha,\theta_i)\geq \widetilde {v}_\beta\},\widehat{A_i}=\{\widehat C(\alpha,\theta_i)\ge \widehat {v}_\beta\}$ and $A_i\Delta \widehat{A_i}:= (A_i\backslash \widehat{A_i})\cup(\widehat{A_i}\backslash A_i) $. 
Then
\begin{equation*}
\begin{split}
\|R_2\|_2&\le \frac{1}{n(1-\beta)}\sum_{i=1}^n  \|\nabla C(\alpha,\theta_i )\|_2 \cdot \mathbbm{1}_{A_i\Delta \widehat{A_i}}\\
&\le \frac{1}{n(1-\beta)} \sum_{i=1}^n B \mathbbm{1}_{A_i\Delta \widehat{A_i}},
\end{split}
\end{equation*}
and
\begin{equation*}
    \begin{split}
        \|R_2\|_2^2
        &\le \frac{1}{n^2(1-\beta)^2} (\sum_{i=1}^n  B \mathbbm{1}_{A_i\Delta \widehat{A_i}})^2\\
        &\le \frac{1}{n(1-\beta)^2}B^2\sum_{i=1}^n\mathbbm{1}_{A_i\Delta \widehat{A_i}}.
    \end{split}
\end{equation*}
Notice that
\begin{equation*}
\mathbb{P}(\mathbbm{1}_{A_i\Delta \widehat{A_i}})
    =\mathbb{P}(A_i\backslash \widehat{A_i})+\mathbb{P}( \widehat{A_i}\backslash A_i).
\end{equation*}
As the estimation error of $\lambda$, i.e. $\|\hat{\lambda}-\lambda\|_\infty $,  is bounded by $\epsilon_\lambda$ and $F $ is $L_{F,\infty}$-Lipschitz continuous w.r.t $\|\cdot\|_\infty$, we have $|\widehat{C}(\alpha,\theta_i)-C(\alpha,\theta_i)|\le L_{F,\infty}\epsilon_\lambda$.  As a result,
$|\widetilde{v}_\beta-\widehat {v}_\beta|\le L_{F,\infty} \epsilon_\lambda$. Notice that
$
\{C(\alpha,\theta_i)\ge \widetilde{v}_\beta+2L_{F,\infty}\epsilon_\lambda\}\subseteq\{\widehat{C}(\alpha,\theta_i)\ge \widehat{v}_\beta\}\subseteq\{C(\alpha,\theta_i)\ge \widetilde {v}_\beta-2L_{F,\infty}\epsilon_\lambda\}.
$
Then we have
$
\mathbb{P}(A_i\backslash \widehat{A_i})+\mathbb{P}( \widehat{A_i}\backslash A_i)\le 4 \ell_C L_{F,\infty} \epsilon_\lambda,
$ by the assumption on the cumulative distribution function of $C$, and thus
\begin{equation}
\label{eq:r2}
\mathbb{E}\|R_2\|_2^2
\le \frac{4}{(1-\beta)^2} B^2 \ell_CL_{F,\infty}\epsilon_\lambda
=\mathcal{O}(\epsilon_\lambda).
\end{equation}

Combining \eqref{eq:r1}, \eqref{eq:r3} and \eqref{eq:r2}, we have 
\begin{equation*}
\mathbb{E}\|[g-\widehat{g}]\|_2^2 \le  \mathcal{O}\left(dn^{-1}+\epsilon_\lambda+\frac{d\epsilon_\lambda^2}{\nu^2}+\frac{d+\nu^2d^3}{m}\right).
\end{equation*}
\end{proof}

Theorem \ref{thm:error_estimator} implies that the error of the gradient estimator can be reduced to arbitrarily small by increasing the sample size $n, m$ or decreasing the truncation error $\epsilon_\gamma$.

\section{Implementing Details}\label{appendix:implementation}
\textbf{Frozen lake problem}. Consider moving from the Start (S) to the Goal (G) on an $5\times 5$ frozen lake with $6$ holes (H). Then there are $18$ ices (F) (involving Start). The agent may not move in the intended direction as the ice is slippery. The position is the row-column coordinate $(i,j)$ with $i,j\in\{0,1,2,3,4\}$ and the state is the $5*i+j$. The state space is $\{0,1,\dots,24\}.$ The action set consists of moving in four directions. The unknown slippery probability is $\theta_s$. Before reaching the goal and standing on the ice, the agent may move in the intended direction with unknown probability $1-\theta_s$ and move in either perpendicular direction with probability $\theta_s/2$. When falling into the hole, the agent may try to escape from the hole and move to the intended direction. Each time the agent will succeed in escaping from the hole with unknown probability $\theta_e$. After reaching the Goal, the agent will always stay in the Goal whatever the action is. We set the cost to be $1$ for each action on ice before reaching goal. Also, stronger efforts may be made when it is harder to escape from the hole. So we set the per-action cost in hole to be uniformly distributed between $[1,1+2(1-\theta_e)]$. We aim to find a policy with the minimum general loss function.
The data set consists of $N$ historical slippery movements and escapement trials. 

\textbf{Linear Loss.} For each of the considered formulations, we obtain the corresponding optimal policy for the same data set and evaluate the actual performance of the obtained policy on the true system, i.e. MDP with the true parameter $\theta^*$. Specifically, we use the  linear loss  function, which corresponds to the total discounted cost in a classical MDP problem.
This is referred to as one replication, and we repeat the experiments for 50 replications using different independent data sets.{ As the random sampling of output policy in Algorithm \ref{alg:sgd_coherent} is for the purpose of proof, we just choose $\alpha_{t_N}$ as the output for convenience in implementation.}
Results for the frozen lake problem  are presented in Table \ref{table: frozen lake small,theta0.02},  with varying data size $N=5$ and $N=50$, slippery probability $\theta_s=0.3$ and escape probability $\theta_e=0.02$. Note that we report the positive-sided variance, which corresponds to the second order moment of the positive component of the difference between the actual loss and the expected loss. Intuitively, a high positive-sided variance indicates more replications with higher costs than the average, which is undesirable.

\textbf{Episodic Case}. We consider the episodic setting where the data collection and policy update are alternatively conducted. Similar with the previous case with fixed data size, we consider the mean loss function with slippery probability $\theta_s=0.3$, escape probability $\theta_e=0.02$, and $ 5\times 20$, $10\times 10$, $20\times 5$ iterations in total. We repeat the experiments for 50 replications on different independent data sets.
Figure~\ref{fig:eps_linear_theta002} shows the decrease of the loss function by different methods.{ As the random sampling of output policy in Algorithm \ref{alg:sgd_episodic} is for the purpose of proof, we just choose $\alpha_{i+1,0}=\alpha_{i,t_i}$ for convenience in implementation.}

Results for the frozen lake problem with escape probability $\theta_e=0.7$ can be found in Table \ref{table: frozen lake small,theta0.7} and Table \ref{table: frozen lake large,theta0.7}.

\begin{table}[!ht]
\centering
\caption{Results for frozen lake problem. Expected loss and positive-sided variance at different risk levels $\alpha$ are reported for different algorithms. Standard errors are reported in parentheses. Escape probability $\theta_e=0.7$ and number of data points is $N=5$.}
\resizebox{11cm}{!}{% 
\begin{tabular}{ccccc}
\hline
\multirow{2}{*}{Approach} & \multicolumn{2}{c}{loss function: mean}             
%& \multicolumn{2}{c}{loss function: mean+variance}       
\\ \cline{2-3} 
& expected loss & positive-sided variance 
    %& expected loss & positive-sided variance 
                          \\ \hline
BR-PG ($\beta=0$)      & {10.322
(0.0182)}  & {0.0153}    
%&    \textcolor{blue}{10.546(0.0404)}             &   \textcolor{blue}{0.073}                  
\\ \hline
    BR-PG ($\beta=0.5$)       & 10.520(0.105)    
    &   0.502
%     & 10.695(0.128)    &   0.743                        
\\ \hline
BR-PG ($\beta=0.9$)       &   11.718
(0.357)  &    4.982             
%&  11.647(0.359)    &  5.237     
\\ \hline
Empirical       &   11.667
(0.0687)   &   0.156
           
           %&   11.307(0.0759)   &  0.158              
           \\ \hline
DRQL  (radius=0.05)                 & 11.223(0.185)     & 1.283    \\ \hline
DRQL  (radius=1)                 & 20.751(1.438)     & 69.514     \\ \hline
DRQL  (radius=20)                 & 23.181(1.396)     & 57.495     \\ \hline
\end{tabular}
}
\label{table: frozen lake small,theta0.7}
\end{table}

\begin{table}[!ht]
\centering
\caption{Results for frozen lake problem. Expected loss and positive-sided variance  at different risk levels $\alpha$ are reported for different algorithms. Standard errors are reported in parentheses. Escape probability $\theta_e=0.7$ and number of data points is $N=50$.}
\resizebox{11cm}{!}{% 
\begin{tabular}{cccc}
\hline
\multirow{2}{*}{Approach} & \multicolumn{2}{c}{loss function: mean}             
%& \multicolumn{2}{c}{loss function: mean+variance}       
\\ \cline{2-3} 
                          & expected loss & positive-sided variance 
                          %& expected loss & positive-sided variance 
                          \\ \hline
BR-PG ($\beta=0$)   & {10.271
(0.00227)}  &   {0.000197}      
%& \textcolor{blue}{10.412(0.00463)}   &   \textcolor{blue}{0.000835}                   
\\ \hline
    BR-PG ($\beta=0.5$)       &   10.256
(0.00211)   &     0.000188
      %& 10.359(0.00432)     &  0.000715                
      \\ \hline
BR-PG ($\beta=0.9$)       &   10.230(0.00294)   &  0.000398
                %&  10.282(0.00538)   &   0.00114      
                \\ \hline
Empirical       &   11.316
(0.0235)  &    0.017
        %&  10.802(0.0456)   &   0.0692                 
        \\ \hline
DRQL  (radius=0.05)                 & 10.888( 0.171)     & 1.235     \\ \hline
DRQL  (radius=1)                 & 20.990( 1.324)     & 56.027     \\ \hline
DRQL  (radius=20)                 & 23.500(1.282)     & 51.915    \\ \hline
\end{tabular}
}
\label{table: frozen lake large,theta0.7}
\end{table}

Figure \ref{fig:map} shows the map of the frozen lake problem with $1$ Start(S), $1$ Goal(G), $6$ holes(H) and remaining frozen(F) parts. We design such a map so that the agent has to avoid falling in the hole when the escape probability is very small and cross the hole when the escape probability is high.  Detailed parameters are set as follows. The true slippery probability  is $ 0.3$. The iteration number for gradient descent is $100$, the stepsize is $0.5$, {and the sample number in each iteration is $r=30$}. we set the discounter factor to be $\gamma=0.97$ , the truncation horizon for occupancy measure to be $K=130$. 
% Also, the sample number for zeroth-order method $m_i=30$ and parameter $\nu$ is $0.5$ in 
\eqref{eq:gradient_sample}.

For the "mean" loss function, we use the maximum likelihood estimator (MLE) of $\theta$  as the empirical measure  to be compared with BR-PG. Also, we use the distributionally robust Q-learning (DRQL)\citep{liu2022distributionally} with different radius for the KL divergence ball as another benchmark. We also use the MLE of $\theta$ as the parameter for the center of  the KL divergence ball in DRQL with different radius. For BR-PG, the sample number from posterior in each iteration is $30$,  the total iteration number is $100$, the step size of SGD is chosen to be $1$, and the prior distributions are chosen to be Beta$(1,1)$ for two parameters.
%First, we show the convergence speed for one replication in Figure \ref{fig:convergence} with escape probability$\theta_e=0.02$ and data size $N=5$. It is obvious that the total cost decreases quickly in the first $20$ iterations and then converge. 
We show the  histogram of total cost over $50$ replications for all methods in Figure \ref{fig:histogram_theta0.02} with the risk level $0.8$ for CVaR over replications, which visualize the measures of dispersion.

\textbf{Mimicking a policy}. Here we consider a different problem of mimicking an expert policy still using  Frozen Lake environment. Given an expert policy, we have access to the state distribution of the expert policy under the true environment, which is denoted by a nonnegative function $J$
%, i.e.$J(s)\ge 0,\forall s \in \mathcal{S}$ 
satisfying $\sum_{s \in \mathcal{S}}J(s)=1$. The loss function we want to minimize is defined as the KL divergence between state occupancy measure under the current policy and the expert state distribution  $F(\lambda)
=\mathrm{KL}\left((1-\gamma)\sum_{a\in\mathcal{A}} \lambda_a||J\right)
=\sum_{s \in \mathcal{S}}\sum_{a\in\mathcal{A}}(1-\gamma) \lambda_{sa} \log\left(\frac{\sum_{a\in\mathcal{A}}(1-\gamma) \lambda_{sa} }{J(s)}\right)$. We compare the BR-PG algorithm with CVaR risk measure under different risk levels $\beta=0,0.5,0.9$, respectively, with the benchmark empirical approach using the MLE estimator for the parameter as before. Figure~\ref{fig:KL} shows the decrease of the loss function by different methods. 
It should be noticed that DRQL can only be applied to the "mean" loss function, thus we don't use it as a benchmark. The performance of the 50 replications is shown in figure~\ref{fig:KL_diff_theta}, where the shown results start from the 30-th iteration.

\begin{figure}[ht]
\centering 
\begin{minipage}[b]{0.45\textwidth} 
\centering 
\includegraphics[scale=0.5]{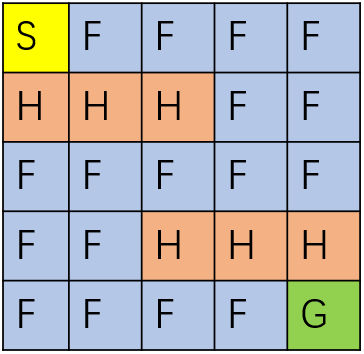}.
\caption{Map of frozen lake problem}
\label{fig:map}
\end{minipage}
% \begin{minipage}[b]{0.45\textwidth} 
% \centering 
%  \includegraphics[scale=0.4]{Figure/Linear_Convergence_DRQL.png}.
% \caption{Convergence result for BR-PG}
% \label{fig:convergence}
% \end{minipage}
\end{figure}

\begin{figure}
     \centering
     \begin{subfigure}[b]{0.4\textwidth}
         \centering
         \includegraphics[width=\textwidth]{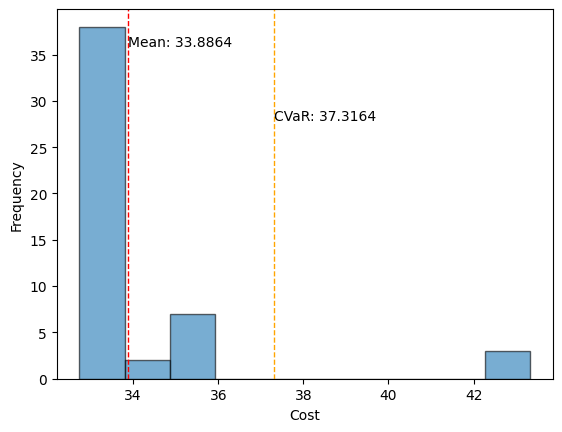}
         \caption{BR-PG  $(\beta=0)$}
         \label{fig:LinearN5Beta0his}
     \end{subfigure}
     \hfill
     \begin{subfigure}[b]{0.4\textwidth}
         \centering
         \includegraphics[width=\textwidth]{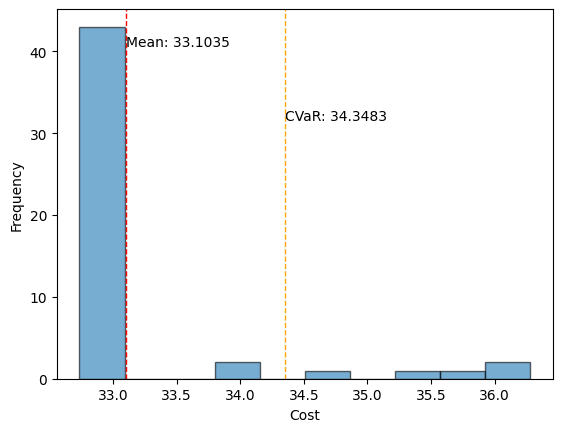}
         \caption{BR-PG  $(\beta=0.5)$}
         \label{fig:LinearN5Beta05his}
     \end{subfigure}
     \hfill
     \begin{subfigure}[b]{0.4\textwidth}
         \centering
         \includegraphics[width=\textwidth]{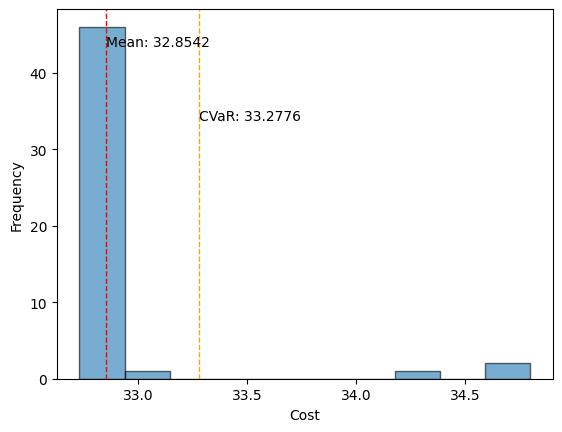}
         \caption{BR-PG  $(\beta=0.9)$}
         \label{fig:LinearN5Beta09his}
     \end{subfigure}
          \hfill
    \begin{subfigure}[b]{0.4\textwidth}
         \centering
         \includegraphics[width=\textwidth]{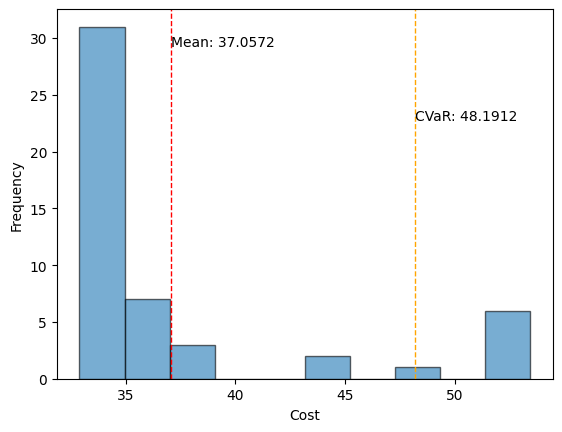}
         \caption{empirical}
         \label{fig:LinearN5MLEhis}
    \end{subfigure}
    %       \hfill
    % \begin{subfigure}[b]{0.4\textwidth}
    %      \centering
    %      \includegraphics[width=\textwidth]{Figure/LinearN5DRQL.png}
    %      \caption{DRQL}
    %      \label{fig:LinearN5DRQLhis}
    %  \end{subfigure}
        \caption{Result for  utility function "mean" with data size $N=5$ and escape probability $\theta_e=0.02$}
        \label{fig:histogram_theta0.02}
\end{figure}

\begin{figure}
     \centering
     \begin{subfigure}[b]{0.4\textwidth}
         \centering
         \includegraphics[width=\textwidth]{Figure/KLN5theta02start30.png}
         \caption{ $\theta_e=0.2$}
         \label{fig:KLN5theta02start30}
     \end{subfigure}
     \hfill
     \begin{subfigure}[b]{0.4\textwidth}
         \centering
         \includegraphics[width=\textwidth]{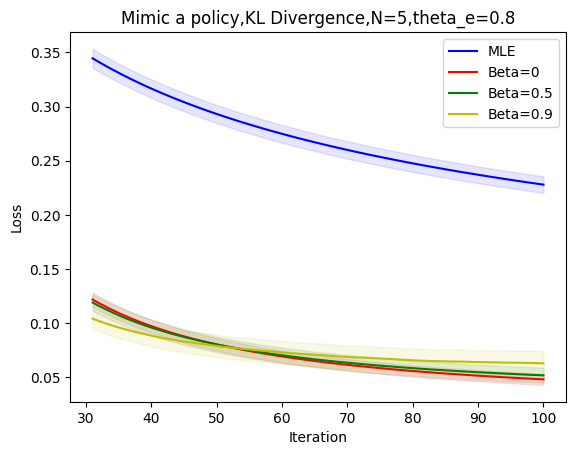}
         \caption {$\theta_e=0.8$}
         \label{fig:KLN5theta08start30}
     \end{subfigure}
        \caption{Results for  utility function "KL divergence" with data size $N=5$ and escape probability $\theta_e=0.2$ and $\theta_e=0.8$}
        \label{fig:KL_diff_theta}
\end{figure}

\end{document}